\documentclass{article}
\usepackage[a4paper,margin=1in]{geometry} %

\usepackage{authblk}
\usepackage{lmodern}
\usepackage{enumitem, amsfonts, amsmath, amssymb, amsthm}
\usepackage[hidelinks]{hyperref}
\usepackage{pifont}
\usepackage{mathrsfs}
\usepackage{bbm}
\theoremstyle{definition}      %
\newtheorem{theorem}{Theorem}
\newtheorem{assumption}{Assumption}
\newtheorem{condition}{Condition}
\newtheorem{remark}{Remark}
\newtheorem{lemma}{Lemma}
\newtheorem{corollary}{Corollary}
\newtheorem{proposition}{Proposition}
\newcommand{\varphilin}{\varphi_{\mathrm{lin}}}
\newcommand{\BG}{\bm{G}}
\newcommand{\Be}{\bm{\epsilon}}
\newcommand{\Bf}{\bm{f}}
\newcommand{\sF}{\mathscr{F}}
\newcommand{\sL}{\mathscr{L}}
\newcommand{\Lipl}{\Lip_{\ell,\cC}}
\newcommand{\Lipf}{\Lip_{\varphi,\cC}}
\newcommand{\Lipdf}{\Lip_{\nabla\varphi,\cC}}
\newcommand{\bv}{\bm\varphi}
\usepackage{amsmath, amssymb}
\usepackage{bm}

\newcommand{\bE}{\mathbb{E}}

\newcommand{\bN}{\mathbb{N}}

\newcommand{\bP}{\mathbb{P}}

\newcommand{\bR}{\mathbb{R}}
\newcommand{\bS}{\mathbb{S}}

\newcommand{\bZ}{\mathbb{Z}}

\newcommand{\cB}{\mathcal{B}}
\newcommand{\cC}{\mathcal{C}}
\newcommand{\cD}{\mathcal{D}}
\newcommand{\cE}{\mathcal{E}}
\newcommand{\cF}{\mathcal{F}}

\newcommand{\cH}{\mathcal{H}}

\newcommand{\cM}{\mathcal{M}}
\newcommand{\cN}{\mathcal{N}}
\newcommand{\cO}{\mathcal{O}}
\newcommand{\cP}{\mathcal{P}}

\newcommand{\cR}{\mathcal{R}}
\newcommand{\cS}{\mathcal{S}}

\newcommand{\Z}{\mathscr{Z}}

\newcommand{\BH}{\bm{H}}
\newcommand{\bbh}{\bar{\bm{H}}}
\newcommand{\s}{\bm{s}}

\newcommand{\Lip}{\rr{Lip}}

\newcommand{\er}{\widehat{\cR}_{\cS}}

\newcommand{\rr}[1]{\mathrm{#1}}

\newcommand{\z}{\bm{z}}

\usepackage{tikz}
\usetikzlibrary{arrows.meta,calc,positioning}

\author{Semih Cayci}
\affil{\small Department of Mathematics\\RWTH Aachen University\\
\texttt{cayci@mathc.rwth-aachen.de}}
\date{}
\usepackage[round]{natbib}

\title{{Non-Asymptotic Optimization and Generalization Bounds for Stochastic Gauss–Newton in Overparameterized Models}}

\begin{document}
\maketitle
\begin{abstract}
An important question in deep learning is how higher-order optimization methods affect generalization. In this work, we analyze a stochastic Gauss-Newton (SGN) method with Levenberg–Marquardt damping and mini-batch sampling for training overparameterized deep neural networks with smooth activations in a regression setting. Our theoretical contributions are twofold. First, we establish finite-time convergence bounds via a variable-metric analysis in parameter space, with explicit dependencies on the batch size, network width and depth. Second, we derive non-asymptotic generalization bounds for SGN using uniform stability in the overparameterized regime, characterizing the impact of curvature, batch size, and overparameterization on generalization performance. Our theoretical results identify a favorable generalization regime for SGN in which a larger minimum eigenvalue of the Gauss–Newton matrix along the optimization path yields tighter stability bounds.

\end{abstract}

\section{Introduction}
In overparameterized models, the training algorithm plays a central role in generalization by implicitly selecting a particular interpolating hypothesis that both fits the data and generalizes well \citep{vardi2023implicit, bartlett2021deep}. A core theoretical and practical challenge in deep learning is to understand this phenomenon. Curvature-aware training algorithms, such as Gauss--Newton (GN) and its variants, have gained increasing attention in deep learning \citep{abreu2025potential,liu2025adam,martens2010deep, lecun2002efficient, botev2017practical, martens2015optimizing}, reinforcement learning \citep{kakade2001natural}, and scientific machine learning \citep{hao2024gauss, mckay2025near, muller2023achieving} due to their ability to achieve robust and fast convergence even in ill-conditioned settings. Despite their empirical success, a concrete theoretical understanding of these second-order methods in deep learning, particularly in the non-asymptotic regime, remains largely elusive. In particular, a non-asymptotic (finite-time and finite-width) theoretical characterization of both the global convergence and the generalization behavior of curvature-aware preconditioned optimization remains an important open problem in deep learning \citep{martens2020new, amari2020does}.

In this work, we analyze the convergence and generalization errors of the stochastic Gauss-Newton method, equipped with empirical Gauss-Newton preconditioner, Levenberg-Marquardt damping, mini-batch sampling and metric projection. Two main questions that motivate our work are as follows:
\begin{itemize}
    \item[\ding{118}] \textbf{Convergence:} \textit{Can we establish non-asymptotic (finite-time) training error bounds for \textsc{SGN}?}
    \item[\ding{118}] \textbf{Generalization:} \textit{How does Gauss-Newton preconditioning impact generalization performance of the resulting predictor?}
\end{itemize}
The use of data- and time-dependent, stochastic and curvature-aware preconditioners, which are correlated with the stochastic gradient, constitutes the main challenge in the analysis of \textsc{SGN}. In our analysis, we interpret SGN as a variable-metric method akin to AdaGrad and Adam \citep{duchi2018introductory}, and adopt a parameter space analaysis with respect to a Lyapunov function that fits the trajectory- and data-dependent Bregman divergence.

\subsection{Main Contributions}
In this work, we establish non-asymptotic optimization and generalization guarantees for SGN for deep and wide feedforward networks in the near-initialization regime. To the best of our knowledge, this is the first theoretical study that establishes stability-based (path-dependent) generalization bounds for Gauss-Newton in deep learning. Our main contributions include the following:

    \paragraph{\textbullet\hskip 3pt Non-asymptotic convergence guarantees for \textsc{SGN}.} We prove that the stochastic Gauss-Newton method with Levenberg-Marquardt damping, mini-batch sampling and metric projection achieves the non-asymptotic convergence rate
\[
  \mathcal{O}\!\left(
    \frac{1}{k}\left[\bar{r}_k\log k+\lambda+\lambda^{-1}\right]
    \;+\;
    \frac{1}{\sqrt{m}}
  \right),
\]
where $k$ is the number of iterations, $B$ the batch size, $m$ is the network width, $\lambda$ is the Levenberg-Marquardt damping factor, and $\bar{r}_k \leq p$ is the intrinsic dimension of the average Jacobian covariance for a $p$-dimensional parameter space.

\paragraph{\textbullet\hskip 3pt Robustness to ill-conditioning.} Our \emph{optimization} bounds (see Section \ref{sec:opt}) do not require strict positive definiteness of the neural tangent kernel, indicating that SGN remains effective even under highly ill-conditioned kernels unlike existing analyses of GN methods. Moreover, the optimization guarantees avoid any polynomial dependence of the network size $m$ on the sample size $n$.

\paragraph{\textbullet\hskip 3pt Generalization bounds for \textsc{SGN}.} We establish non-asymptotic and \emph{algorithm-dependent} generalization bounds for \textsc{SGN} via the concept of algorithmic stability in the sense of \citep{bousquet2002stability}, and show that \textsc{SGN} is uniformly stable under appropriately-chosen damping and learning rate choices, overparameterization and kernel non-degeneracy (see Theorem \ref{thm:stability}). Here, the damping factor $\lambda$ governs a trade-off between optimization and generalization performance. We identify benign cases, where persistence of excitation (i.e., well-conditioned preconditioner across iterations) implies a favorable generalization and optimization performance. Interestingly, our theoretical analysis shows that increasing the model complexity (network width) implies improved generalization bounds in the overparameterized regime, aligning with the empirical observations. On the technical side, stability analysis of \textsc{SGN} is challenging due to its data- and path-dependent preconditioner. We address this by establishing approximate co-coercivity and metric perturbation bounds in the wide-network regime.

\subsection{Related Works}
\paragraph{Convergence of GN in Deep Learning.} Gauss–Newton and its variants are classical tools for nonlinear least squares \citep{nocedal2006numerical,gratton2007approximate,bertsekas1996incremental} and have inspired practical second-order routines in deep learning \citep{liu2025adam, botev2017practical, abreu2025potential}. However, a concrete theoretical understanding of GN methods in deep learning is still in a nascent stage, and there has been a surge of interest recently to analyze the convergence of Gauss-Newton variants in deep learning \citep{cayci2024riemannian, arbel2023rethinking, cai2019gram, zhang2019fast}, following the NTK analyses in \citep{chizat2019lazy, du2018gradient}. These works typically require (i) a lower bound on the NTK/Gram spectrum, and (ii) polynomial overparameterization $m = \Omega(n^q),~q\geq 2$ for convergence bounds. Our optimization results establish non-asymptotic convergence bounds for stochastic Gauss–Newton (i) without assuming NTK conditioning (i.e., data separation), and (ii) without massive overparameterization. On the technical side, to capture stochastic and path-dependent metric in \textsc{SGN}, we devise a new variable-metric analysis in parameter space for Gauss-Newton iterations, while existing works \citep{cai2019gram, zhang2019fast, arbel2023rethinking, cayci2024riemannian} utilize a function space analysis.

\paragraph{Generalization of Gauss-Newton.} The impact of curvature-aware preconditioning on generalization in deep learning has been a focal point of interest \citep{amari2020does, zhang2019fast, arbel2023rethinking}. The implicit bias of Gauss-Newton was studied by \cite{arbel2023rethinking} empirically without any quantitative generalization bounds. The existing generalization bounds are either (i) algorithm-agnostic (e.g., bounds based on Rademacher complexity that solely depend on the hypothesis class rather than the training rule \citep{zhang2019fast}), or (ii) asymptotic comparisons that analyze the risk of the limiting solution via bias–variance decompositions (in ridgeless regression/RKHS) \citep{amari2020does}, and therefore do not yield finite-sample, finite-time bounds that track optimization trajectories in deep networks. Unlike the existing works, we derive generalization bounds for SGN that \textit{track} the optimization path of SGN.

\paragraph{Algorithmic stability.} Uniform stability framework is due to \cite{bousquet2002stability}, and has been improved and refined by \cite{feldman2019high, klochkov2021stability}. Most applications of the algorithmic stability framework in machine learning focus on first-order methods (e.g., SGD), where stability analyses are used for generalization bounds \citep{hardt2016train,lei2020fine,bassily2020stability, hellstrom2025generalization}. However, these results are not applicable for curvature-aware, preconditioned second-order methods, such as \textsc{SGN}, as the parameters are path-dependent (non-Markovian) and evolve under time- and data-dependent metrics, which we address in this work.

\subsection{Notation}
For $\bm{x}\in\bR^d, \rho > 0$, $\cB_2(\bm{x},\rho)=\{\bm{y}\in\bR^d:\|\bm{x}-\bm{y}\|_2\leq \rho\}$. $d_H(\bm{s},\bm{s}')$ denotes the Hamming distance between two vectors $\bm{s},\bm{s'}\in\bR^n$. For any $n\in\bZ_+$, $[n]:=\{1,2,\ldots,n\}$. For any symmetric positive definite matrix $\bm{M}\in\bR^{n\times n}$ and $w\in\bR^n$, $\|\bm{w}\|_{\bm{M}}:=\sqrt{w^\top \bm{M}w}$, and $\lambda_{\min}(\bm{M})$ and $\lambda_{\max}(\bm{M})$ denotes the minimum and maximum eigenvalue of $\bm{M}$, respectively.
\section{Problem Setup}\label{sec:problem-setup}
\subsection{Supervised Learning Problem}
Let $\mathscr{X}\subset \bR^d$ be the input space, $\mathscr{Y}\subset \bR$ be the output space, and $\Z = \mathscr{X}\times\mathscr{Y}$ be the sample space. Let $\cP$ be the sample distribution over $\Z$, and let $\mu(\cdot)=\int_{\mathscr{Y}}\cP(\cdot,dy)$ be the input distribution. Let $\bm{z}:=(\bm{x}, y)\sim\cP$ be an input-output pair. We assume that $\cP$ has a compact support set such that $$\|\bm{x}\|_2\leq 1\quad \mbox{and}\quad |y|\leq 1,$$ almost surely. Given a predictor $\varphi:\bR^d\rightarrow\bR$, the loss at $\bm{z}=(\bm{x},y)$ is denoted as $\ell(\varphi(\bm{x}), \bm{z})$. The goal in supervised learning is to minimize the population risk \begin{equation}\tag{SL}\cR(\varphi):=\bE_{\bm{z}\sim\cP}\left[\ell(\varphi\left(\bm{x}), \z\right)\right].\label{eqn:sl}\end{equation} 
\begin{assumption}\label{assumption:loss}
    For any $\bm{z}\in\mathscr{Z}$, the loss function $y'\mapsto \ell(y',\bm{z})$ is twice-differentiable, $\nu$-strongly convex for some $\nu > 0$, and (locally) $\beta_K$-Lipschitz within any compact $K\subset\bR$ for some $\beta_K>0$ .
\end{assumption}
\noindent Note that the square loss $$\ell(y',\bm{z}) = \frac{1}{2}(y-y')^2$$ satisfies Assumption \ref{assumption:loss} with $\nu = 1$ and $\beta_K=1+\sup_{u\in K}|u|$. Logistic loss under Tikhonov regularization is another example, $$\ell(y',\mathbf z)=\log\bigl(1+e^{-y\,y'}\bigr)+\frac{\lambda}{2}(y')^2,\quad \lambda>0$$ with $\nu = \lambda$ and $\beta_K = 1+\lambda\sup_{u\in K}|u|$.

For a sequence of independent and identically-distributed sequence $\bm{z}_i:=(\bm{x}_i,y_i)\overset{iid}{\sim}\cP,~ i\in[n]$, we denote the training set by $\cS:=\{\bm{z}_i:i\in[n]\}$ and define the empirical risk as $$\widehat{\cR}_{\cS}(\varphi):=\frac{1}{n}\sum_{i=1}^n\ell(\varphi(\bm{x}_i), \bm{z}_i).$$ Given a parametric class $\cF_\Theta:=\{\varphi_\theta:\theta\in\Theta\}$ of predictors, the empirical risk minimization (ERM) is the problem of solving
\begin{equation}\tag{ERM}
    \min_{\theta\in\Theta}~\er(\varphi_\theta).
    \label{eqn:erm}
\end{equation}
A fundamental quantity in solving \eqref{eqn:sl} is the generalization error \citep{hellstrom2025generalization}:
\begin{equation}
    \mathrm{gen}(\varphi;\cS):=\cR(\varphi)-\er(\varphi).
\end{equation}
An empirical risk minimizer with a low generalization error effectively solves \eqref{eqn:sl} by achieving low population risk \citep{hellstrom2025generalization}. For a solution $\widehat{\theta}_\cS$ to \eqref{eqn:erm}, the fundamental challenge is to obtain a bound on $\bE_{\cS\sim\cP^n}[\mathrm{gen}(\varphi_{\widehat{\theta}_{\cS}};\cS)].$ Traditionally, uniform convergence bounds with Rademacher complexity \citep{bartlett2002rademacher} and VC dimension \citep{anthony2009neural} have been used to establish generalization bounds. However, these bounds are \emph{algorithm-agnostic} (i.e., insensitive to the learning algorithm), and therefore fail to capture essential dynamics under overparameterization \citep{zhang2021understanding}.

\subsection{Deep Feedforward Neural Networks}
We consider a deep feedforward neural network of depth $H \geq 1$ and width $m$ with a smooth activation function $\sigma:\bR\rightarrow\bR$, which satisfies $$\|\sigma\|_\infty\leq \sigma_0,~\|\sigma'\|_\infty\leq \sigma_1,\mbox{ and }\|\sigma''\|_\infty\leq \sigma_2.$$ Note that $\tanh$ (with $\sigma_0=1,\sigma_1=2,\sigma_2=2$) and sigmoid function satisfy this condition.

We consider the following model: let $W^{(1)}\in\bR^{m\times d}$ and $W^{(h)}\in\bR^{m\times m}$ for $h=2,3,\ldots,H$, and $$\bm{W}:=(W^{(1)},\ldots,W^{(H)}).$$ Then, given a training input $\bm{x}_j\in\bR^d$, the neural network is defined recursively as
\begin{align*}
    \bm{x}_j^{(h)}(\bm{W})&=\frac{1}{\sqrt{m}}\cdot\vec{\sigma}\Big(W^{(h)}\bm{x}^{(h-1)}_j(\bm{W})\Big),\quad h\in[H],\\
    \varphi(\bm{x}_j;w)&=c^\top\bm{x}_j^{(H)}(\bm{W})
\end{align*}
where $\bm{x}_j^{(0)}(\bm{W})=\bm{x}_j$, $w={\rr{vec}}(\bm{W},c)$ is the parameter vector, and $\vec{\sigma}(z)=[\sigma(z_1)~\ldots~\sigma(z_m)]^\top$. This parameterization, along with $\cO(1/\sqrt{m})$ scaling, is common in the neural tangent kernel literature \citep{du2018gradient, chizat2019lazy}. We note that our analysis extends to general output scaling factors $\mathcal{O}(1/m^\zeta)$ for any $\zeta > 0$ except Corollary \ref{cor:ntk} that necessitates $\zeta =0.5$.

Let $p$ be the number of parameters of the neural network. The following elementary result will be useful in both the generalization and optimization bounds.
\begin{lemma}\label{lemma:continuity}
    For any compact and convex $\cC\subset \bR^p$ with 
    \begin{equation}\label{eqn:radius}
    \sup_{w,w'\in\cC}\|w-w'\|_2 =: r_{\cC},
    \end{equation}
    let $\kappa_0:=\max_{h\in[H]}\frac{\|W_0^{(h)}\|_2}{\sqrt{m}}\mbox{ and }\zeta_0:=\|c_0\|.$ Also, let $\kappa_\cC = \kappa_0+\frac{r_{\cC}}{\sqrt{m}}$ and $\zeta_\cC:=\zeta_0+\frac{r_{\cC}}{\sqrt{m}}$. We have the following (local) Lipschitz continuity results in $\cC$:
    \begin{enumerate}[label=(\roman*)]
        \item For any $H \geq 1$ and $\|\bm{x}\|_2\leq 1$,
        \begin{equation}
            \sup_{w\in\cC}\|\nabla_w\varphi(\bm{x};w)\|_2 \leq \sigma_0 + \zeta_{\cC}\frac{\sigma_0\sigma_1}{\sqrt{m}}\sqrt{H}(\sigma_1\kappa_{\cC})^{H-1}
        \end{equation}
        \item For any $H \geq 1$ and $\|\bm{x}\|_2\leq 1$,
        \begin{align*}
            \sup_{w\in\cC}\|\nabla_w^2\varphi(\bm{x};w)\|_2 &\leq 8(\sigma_2\sigma_0+\sigma_1^{2}\zeta_{\mathcal{C}})\Big[\frac{H\kappa_\cC^{H-1}}{\sqrt{m}}+\Big(\frac{H\kappa_\cC^{H-1}}{\sqrt{m}}\Big)^2\Big]\\
            &=:\rr{Lip}_{\nabla\varphi,\cC}
        \end{align*}
    \end{enumerate}
\end{lemma}
Under Assumption \ref{assumption:loss}, we define $\Lip_{\ell,\cC}:=\beta_K$ where $K=[0,\sup_{w\in\cC,\bm{x}\in\cB_2(0,1)}|\varphi(\bm{x},w)|]$.

\section{Stochastic Gauss-Newton for Empirical Risk Minimization}\label{sec:sgn}
Let $\cC\subset\bR^p$ be a closed and convex subset. For a fixed batch-size $B\leq n$, let $\{I_k:k\in\bN\}$ be an independent and identically distributed (i.i.d.) random process such that $\bP(I_k=\cM)=B/n$ for any $\cM\subset[n]$ with $|\cM|=B$. Let $$\bm{J}_k(w):=[\nabla_w^\top \varphi(\bm{x}_j;w)]_{j\in I_k}\in\bR^{B\times p},~k\in\bN$$ be the Jacobian of the predictor, $$G_k(w):=[\ell'(\varphi(\bm{x}_j;w);\bm{z}_j)]_{j\in I_k}\in\bR^B,$$ be the gradient of the loss function, and $$\Phi_k(w):=\sum_{j\in I_k}\ell(\varphi(\bm{x}_j;w);\bm{z}_j)$$ be the loss evaluated at $I_k$. Then, the gradient is $$\Psi_k(w) = \nabla \Phi_k(w)= \bm{J}_k^\top (w)G_k(w).$$ Given $\lambda, \alpha > 0$, we consider the following preconditioner, which is an incremental version of the Gauss-Newton preconditioner under stochastic subsampling \citep{bertsekas1996incremental, martens2020new}: $$\bm{H}_k=\alpha \sum_{t=0}^k\bm{J}_t^\top(w_t)\bm{J}_t(w_t)+\lambda\bm{I}.$$ For an initial parameter $w_0\in\cC$, for any $k=0,1,\ldots$, the (stochastic) Gauss-Newton update is
\begin{align}
    \begin{aligned}
    u_k &= w_k-\eta\bm{H}_k^{-1}\Psi_k(w_k),\\
    w_{k+1} &= \pi_{\cC}^{\bm{H}_k}(u_k),
    \end{aligned}
    \label{eqn:sgn}
\end{align}
where $\pi_{\cC}^{\bm{H}_k}(w) = \arg\min_{u\in\cC}\|u-w\|_{\bm{H_k}}^2$.

\begin{remark}\normalfont 
    The Gauss-Newton method can be interpreted in multiple ways, in both Euclidean and Riemannian geometries \citep{nocedal2006numerical, cayci2024riemannian}. Deviating from the classical analysis approach that interprets Gauss-Newton as an approximate Newton method (see \citep{nocedal2006numerical}), we adopt a variable-metric approach building on \cite{duchi2018introductory, hazan2007logarithmic}:
    \begin{equation*}
        w_{k+1} = \arg\min_{w\in\cC}\left\{\Psi^\top_k(w_k)(w-w_k)+\frac{1}{\eta}\cD_k(w,w_k)\right\},
    \end{equation*}
    where the Bregman divergence is $$\cD_k(w,w_k)=\frac{1}{2}(w-w_k)^\top \bm{H}_k(w-w_k),~k\in\bN.$$ Note that $\cD_k$ is akin to the (generalized) Mahalanobis distance in mirror descent, yet it is more complicated as it is time- and path-dependent to capture curvature information.
        Note that $\bm{H}_k$ is incrementally updated using the Gram matrices $\bm{J}^\top_k(w_k)\bm{J}_k(w_k)$ throughout the optimization path, reminiscent of the adaptive gradient methods such as AdaGrad \citep{duchi2011adaptive} and Shampoo \citep{gupta2018shampoo, morwani2025a}; however, \textsc{SGN} uses an approximation $\bm{J}_k^\top(w_k)\bm{J}_k(w_k)$ of the Hessian $\nabla_w^2\er(\bv(w_k))$ instead of gradient outer products, leading to significantly different dynamics. 
    \end{remark}
\begin{remark}[Beyond quadratic loss]\normalfont

    The stochastic Gauss-Newton preconditioner \( \bm{H}_k \) is derived for the squared loss \( \ell(f,\bm{z})=\tfrac12(f-y)^2 \). Our analysis in this paper applies to a broader class of loss functions that satisfy Assumption~\ref{assumption:loss}, which includes the squared loss (\(\nu=1\)). Extending the theory to more general preconditioners that incorporate \( \nabla_f^2 \ell(f,\bm{z}) \big|_{f=\varphi(\bm{x};\bm{w})} \) is an interesting direction for future work.
\end{remark}

\section{Finite-Time Optimization Bounds for SGN}\label{sec:opt}
In this section, we establish finite-time and finite-width optimization bounds for the \textsc{SGN} for \eqref{eqn:erm} with deep networks. In particular, we prove (i) global near-optimality within $\cC$, and (ii) global optimality in the neural tangent kernel regime for the \textsc{SGN} for wide neural networks.
\subsection{Finite-Time Error Bounds for \textsc{SGN}}

In the following, we provide the main optimization results for \textsc{SGN}. Recall that $\cC\subset\bR^p$ is a compact and convex parameter set with diameter $r_{\cC}$. Given a training set $\cS$ and parameter $w\in\bR^p$, we denote the predictions as $\bm{\varphi}(w)=[\varphi(\bm{x}_j;w)]_{j\in[n]}$.

\begin{theorem}[Finite-Time Bounds for \textsc{SGN}]\label{thm:optimization}\normalfont
    Let $$\cE(w_k):=\er(\bm{\varphi}(w_k))-\inf_{w\in\cC}\er(\bm \varphi(w))$$ be the optimality gap under \textsc{SGN}, and
    \begin{equation*}
        \gamma := \frac{\lambda}{\alpha\cdot \Lip_{\varphi,\cC}^2\cdot B}\quad\mbox{and}\quad \xi := \frac{\eta}{\alpha}.
    \end{equation*}
    Then, with any choice of $(\eta,\alpha,\lambda)$ such that 
    $\xi \geq \frac{2}{\nu},$
    \textsc{SGN} achieves
    \begin{align*}
        \frac{1}{k}\sum_{t=0}^{k-1}\bE[\cE(w_t)]&\leq \frac{1}{k}\Bigg[\frac{\Lip_{\varphi,\cC}^2r_{\cC}^2(\gamma+2)}{\xi}+ \xi\Lip_{\ell,\cC}^2\Big(\bE\log\frac{\det\bm{H}_k}{\det\bm{H}_0}+\frac{1}{\gamma+1}\Big)\Bigg]+ \frac{r_{\cC}^4\Lip_{\nabla\varphi,\cC}^2}{\xi}+\Lip_{\ell,\cC}r_{\cC}^2\Lip_{\nabla\varphi,\cC}\\
        &=: \kappa_{k,m}
    \end{align*}
    for any $k,m\in\bZ_+$. Let $\widehat{w}_k:=\frac{1}{k}\sum_{t=0}^{k-1}w_t$ be the Polyak-Ruppert average under SGN. Then,
    \begin{equation*}
        \bE[\cE(\widehat{w}_k)] \leq 4\cdot \kappa_{k,m} + 2\cdot r_{\cC}^4\cdot \underbrace{\Lip_{\nabla\varphi,\cC}^2}_{\lesssim 1/m}
    \end{equation*}
    for any $k \in\bZ_+$.
\end{theorem}

\begin{remark}\normalfont
Note that if $\xi \geq 2$ is satisfied, Theorem \ref{thm:optimization} and Lemma \ref{lemma:continuity} imply a convergence rate
\begin{equation*}
    \bE\left[\cE\left(\frac{1}{t}\sum_{t<k}w_t\right)\right]\lesssim \frac{1}{k}\left(\frac{\gamma}{\xi}+\frac{\xi}{\gamma}+\xi\bar{r}_k\log(Bk)\right)+\frac{1}{m\xi}+\frac{1}{\sqrt{m}},
\end{equation*}
where $\bar{r}_k\leq p$ is characterized in Proposition \ref{prop:continuity}. Unlike \citep{cai2019gram, cayci2024riemannian, zhang2019fast}, Theorem \ref{thm:optimization} does \emph{not} require positive definiteness of the kernel $\bE[\bm{J}_0(w_0)\bm{J}_0^\top(w_0)]$, thus the bounds holds without massive overparameterization $m=\mathrm{poly}(n)$. Also, the bound does not depend on the Lipschitz smoothness of $\ell$, suggesting robustness of SGN in regimes with ill-conditioned kernels.
\end{remark}
The complete proof of Theorem \ref{thm:optimization} is provided in Appendix \ref{app:opt}. On the technical side, our approach deviates significantly from the existing analyses \citep{nocedal2006numerical, cayci2024riemannian, arbel2023rethinking, zhang2019fast}: in order to capture both the stochasticity and the incremental nature of the preconditioner along the optimization path, we analyze the \textsc{SGN} method in parameter space via a time-varying Lyapunov function, $\mathscr{L}_k(w):=\|w-\bar{w}\|_{\bm{H}_k}^2$, bringing analytical tools from adaptive gradient methods \citep{duchi2018introductory, hazan2007logarithmic}, instead of function space analyses \citep{arbel2023rethinking, zhang2019fast, cai2019gram, cayci2024riemannian}. Crucially, the Jacobian-based preconditioner $\bm{H}_k$ necessitates a novel analysis beyond standard adaptive methods.

In the following, we characterize the $\log \det$ term.
\begin{proposition}\label{prop:continuity}
    For any compact and convex $\cC\subset \bR^p$,
        \begin{equation*}
            \log\frac{\det\bm{H}_k}{\det \bm{H}_0}\leq p\log\left( 1+\frac{\alpha}{\lambda}(k+1)\sqrt{B}\rr{Lip}_{\varphi,\cC}\right),
        \end{equation*}
        for all $k\in\bN$ almost surely. Furthermore, let 
        \begin{equation}\label{eqn:intrinsic-rank}
        \bar{\Sigma}_k:=\frac{1}{k}\sum_{t=0}^{k-1}\bE[\Sigma(w_t)] ~\mbox{where}~\Sigma(w):=\frac{1}{n}\sum_{j=1}^n\nabla\varphi(\bm{x}_j;w)\nabla^\top\varphi(\bm{x}_j;w).
        \end{equation}
        Then, for the intrinsic rank $\bar{r}_k:=\mathrm{rank}(\bar{\Sigma}_k) \leq p$, we have
        \begin{equation*}
            \bE\log \frac{\det\bm{H}_k}{\det \bm{H}_0} \leq \bar{r}_k\Big(\log(k+1) + \gamma^{-1}+\frac{\Lip_{\varphi,\cC}^2}{\bar{r}_k}\Big) \lesssim \bar{r}_k\log k.
        \end{equation*}
    
\end{proposition}
\noindent The proof of Proposition \ref{prop:continuity} can be found in Appendix \ref{app:gen}.
\begin{remark}
    Proposition \ref{prop:continuity} and Theorem \ref{thm:optimization} together imply that SGN has an iteration complexity $\widetilde{\mathcal{O}}\left(\bar{r}_k/\epsilon\right)$, and in the worst case scenario, the iteration complexity is $\widetilde{\mathcal{O}}(p/\epsilon)$. In structured problems with low intrinsic rank $\bar{r}_k \ll p$, SGN achieves improved iteration complexity.
\end{remark}

\subsection{Global Near-Optimality in the Neural Tangent Kernel Regime}
Theorem \ref{thm:optimization} indicates that \textsc{SGN} achieves near-optimality within the parameter set $\cC$. The approximation error in $\cC$ is $\inf_{w\in\cC}\er(\bm{\varphi}(w))-\inf_{\bm f\in\bR^n}\er(\bm f)$. In this subsection, we show that for sufficiently large $\cC$, \textsc{SGN} can attain the \emph{globally} optimal predictor for $\er$ up to an arbitrarily small approximation error.

First, we define a rich function class $\cF_{\textsc{ntk}}$ that corresponds to the infinite-width limit of shallow and wide neural networks \citep{cayci2025convergence, Ji2020Polylogarithmic}.
\paragraph{\textbf{NTK Function Class.}} Let
\begin{align*}
\cH := \Big\{(v_c,v_W):\bR^{d+1}\rightarrow \bR\times\bR^d: \sup_{\omega\in\bR^{d+1}}|v_c(\omega)|\leq \bar{v}_c,\sup_{\omega\in\bR^{d+1}}|v_W(\omega)|\leq \bar{v}_W\Big\}.
\end{align*}
Then, for $\bm{v}:=(v_c,v_W)$, we define
\begin{equation*}
    \cF_{\textsc{ntk}} := \{\bm{x}\mapsto \bE_{\omega\sim\cN(0,\bm{I}_{d+1})}\langle\bm{v}(\omega),\phi(\bm{x};\omega)\rangle:\bm{v}\in\cH\},
\end{equation*}
where $$\phi(x;\omega)=\left(\sigma(\langle W,\bm{x}\rangle),~c\bm{x} \sigma'(\langle W, \bm{x}\rangle)\right)$$ is the random feature for $\omega=(c,W)\in\bR^{d+1}$. The completion of $\cF_{\textsc{ntk}}$ is the reproducing kernel Hilbert space (RKHS) of the NTK for a single hidden-layer network \citep{cayci2025convergence, Ji2020Polylogarithmic}, which is dense in the space of continuous functions on a compact set \citep{Ji2020Neural}. 

The following result for single hidden-layer network shows that $f^\star\in\cF_{\textsc{ntk}}$ can be learned by \textsc{SGN} with high probability over the random initialization.

\begin{corollary}[Near-optimality in $\cF_{\textsc{ntk}}$]
    Assume that $$y_j = f^\star(\bm{x}_j),\quad j\in[n],$$ for some $f^\star \in\cF_{\textsc{ntk}}$. Consider the random initialization $w_0=[u_0^{(i)}]_{i\in[m]}$, where 
    \begin{align}
    \begin{aligned}
        &(c_0^{(i)},W_0^{(i)})\overset{iid}{\sim}\cN(0, \bm{I}_{d+1}),\\
        &c_0^{(i+m/2)}=-c_0^{(i)}\quad\mbox{and}\quad W_0^{(i+m/2)}=W_0^{(i)}
        \end{aligned}
        \label{eqn:initialization}
    \end{align}
    for $i=1,2,\ldots,m/2$. For any $\delta \in (0, 1]$, let $$\rho_{\bm{\nu}}:=\|(\bar{v}_c,\bar{v}_W)\|_2,$$ and $\cC := \cB_2\big(w_0, \rho_{\bm{\nu}}\big)\subset\bR^{m(d+1)}.$
    \label{cor:ntk}
\noindent Then, with probability $1-\delta$ over the random initialization, \textsc{SGN} with $\xi=\eta/\alpha \geq 2/\nu$ and $\lambda=\gamma \alpha \Lip_{\varphi,\cC}^2B$ achieves
\begin{equation*}
    \bE_0\er\left(\bm \varphi(\widehat{w}_k)\right) \lesssim C_{n,\delta,\bm{\nu}}\Big(\frac{\bar{r}_k\log k}{k}+\frac{1}{\sqrt{m}}\Big),
\end{equation*}
where $\widehat{w}_k:=\frac{1}{k}\sum_{k=0}^{t-1}w_k$ is the average iterate, $C_{n,\delta,\bm{\nu}}=\rr{poly}\big(\log\left(n/\delta\right),\|\bm{\nu}\|_2\big)$, and $\bE_0[\cdot]:=\bE[\cdot|w_0]$.
\end{corollary}

The proof of Corollary \ref{cor:ntk} with exact constants can be found in Appendix \ref{app:opt}.

\section{Generalization Bounds for SGN via Uniform Stability}\label{sec:gen}
In this section, we prove finite-sample and finite-width generalization bounds for deep neural networks trained by \textsc{SGN} from the perspective of \emph{algorithmic stability} \citep{bousquet2002stability, hardt2016train}. The intuition behind the algorithmic stability is that learning algorithms with a weak dependency on the specific training data (e.g., nearly-insensitive to an arbitrary change in one sample point in the training set) generalize well to test data \citep{hellstrom2025generalization}.

We first provide a concise mathematical description of the algorithmic stability concept, and then present our generalization bounds.

\subsection{Algorithmic Stability}
We first present an overview of algorithmic stability \citep{hardt2016train, bousquet2002stability}, which will constitute the basis of our generalization bounds.

For a given training set $\bm{s}\in \mathscr{Z}^n$, let $A(\bm{s})$ be the $\bR$-valued output of a randomized learning algorithm. If, for a given $\epsilon > 0$,
\begin{equation*}
    \sup_{\bm{s},\bm{s'}:d_H(\bm{s},\bm{s'})=1}~\sup_{\bm{z}\in\mathscr{Z}}~\bE_{A}[\ell(A(\bm{s});\bm{z})-\ell(A(\bm{s'});\bm{z})] \leq \epsilon,
\end{equation*}
then $A$ is called $\epsilon$-uniformly stable. An $\epsilon$-uniformly stable algorithm achieves $
\left|\bE_{A,\cS}\left[\rr{gen}(A(\cS);\cS)\right]\right| \leq \epsilon
$
(Theorem 2.2 in \cite{hardt2016train}). In \textsc{SGN}, the randomness of the learning algorithm stems from the subsampling process $\{I_k:k\in\bN\}$.

\subsection{Uniform Stability of \textsc{SGN}}
Let $\bm{s},\bm{s}'\in\Z^n$ be two arbitrary training sets such that $d_H(\bm{s},\bm{s}')=1$, and let $j^*\in[n]$ denote the sample index where $\bm{s}$ and $\bm{s}'$ differ, i.e., $\bm{z}_{j^\star}\neq \bm{z}_{j^\star}'$. In the following, $\{(w_k,\widehat{w}_k, \Phi_k,\bm{J}_k,G_k,\Psi_k,\bm{H}_k):k\in\bN\}$ and $\{(w_k', \widehat{w}_k',\Phi_k',\bm{J}_k',G_k',\Psi_k',\bm{H}_k'):k\in\bN\}$ denote the parameters and related mappings in \textsc{SGN} trained over $\bm{s}$ and $\bm{s}'$, respectively, using the same trajectory of the subsampling process $\{I_k:k\in\bN\}$ and from the same initial condition $w_0=w_0'$. The following yields a stability bound on the average iterate, which is the output of SGN.
\begin{lemma}[Stability with midpoint metric]
    For any $k\in\bN$, let $\Delta_k := w_k-w_k'$, and define the midpoint metric
    \begin{equation}
        \bar{\BH}_k := \frac{1}{2}(\BH_k+\BH_k').
    \end{equation}
    For any $\s,\s'\in\Z^n$ such that $d_H(\s,\s')=1$, we have
    \begin{align*}
        \sup_{\bm{z}\in\Z}\bE_{(I_t)_{t\leq k}}|\ell(\varphi(\bm{x};\widehat{w}_k);\bm{z})-\ell(\varphi(\bm{x};\widehat{w}_k');\bm{z})| \leq \frac{\Lip_{\ell,\cC}\Lip_{\varphi,\cC}}{\sqrt{\lambda}}\frac{1}{k}\sum_{t=1}^{k}\bE\|\Delta_t\|_{\bar{\BH}_{t-1}},
    \end{align*}
    \label{lemma:stability}
    where $\widehat{w}_k:=\frac{1}{k}\sum_{t<k}w_t$.
\end{lemma}
Throughout the rest of this section, we establish finite-sample bounds on $\bE\|w_k-w_k'\|_{\bar{\BH}_{k-1}}$, which implies the uniform stability of \textsc{SGN} together with Lemma \ref{lemma:stability}.

We make the following assumption for the generalization bounds. Let $\bm{J}(w):=[\nabla^\top\varphi(\bm{x}_j;w)]_{j\in[n]}\in\bR^{n\times p}$.
\begin{assumption}
We assume that there exists $\mu_0 >0$ such that $\bm{J}(w)\bm{J}^\top(w)\succeq \mu_0^2\bm{I}$ for all $w\in\cC$.
\label{assumption:ntk}
\end{assumption}
\noindent Assumption \ref{assumption:ntk} is standard in lazy training: it holds for sufficiently wide networks ($m\gtrsim n^2\log n$), small enough $r_{\cC}$, and an input space $\mathscr{X}\subset \bS^{d-1}$ with non-collinear $\bm{x}\nparallel \bm{x}'$ points with high probability at random initialization \citep{chizat2019lazy, du2018gradient}.

\begin{theorem}[Uniform Stability of \textsc{SGN}]\label{thm:stability}
Let
    \begin{equation*}
        \varepsilon := B\Lip_{\ell,\cC}\Lip_{\nabla\varphi,\cC}\quad\mbox{and}\quad\Lambda := B\Lip_{\varphi,\cC}^2 + \varepsilon,
    \end{equation*}
    and $\{\lambda_t:t\in\bN\}$ be such that $\lambda_{\min}(\bm{H}_t)\geq \lambda_t$ for all\footnote{Since $\lambda_{\min}(\bm{H}_k) \geq \lambda > 0$ for all $k\in\bN$, the existence of such a sequence is guaranteed.} $t \in\bN$. With the choices that satisfy 
    \begin{equation}\label{eqn:hyperpars}
    \frac{\eta}{\lambda} \leq \frac{1}{\Lambda}\quad\mbox{and}\quad\frac{\alpha}{\eta} \leq \frac{\mu_0^2\nu^2}{8B(\Lambda+\varepsilon)},
    \end{equation}
    we obtain, for each $k\in\bZ_+$,
\begin{align*}
    \bE\|\Delta_k\|_{\bar{\BH}_{k-1}} \lesssim \underbrace{k\sqrt{\eta B \Lip_{\nabla\varphi,\cC}}+{k\sqrt{\alpha}B^\frac{3}{2}\Lip_{\nabla\varphi,\cC}}}_{\mbox{\footnotesize \normalfont approximate non-expansivity}}
    + \underbrace{\alpha B^2\Big(\Lip_{\nabla\varphi,\cC}+\frac{1}{n}\Big)\sum_{t=0}^{k-1}\frac{t+1}{\lambda_t^\frac{1}{2}}}_{\normalfont \mbox{\footnotesize metric/preconditioner mismatch}}+\underbrace{\frac{\eta B}{n}\sum_{t=0}^{k-1}\lambda_t^{-\frac{1}{2}}}_{\normalfont\mbox{\footnotesize gradient mismatch}}
\end{align*}
\end{theorem}
A proof sketch for Theorem~\ref{thm:stability} is provided in Section~\ref{subsec:stability-analysis}, where each term above is explained. We have the following result since $\inf_{t\geq 0}\lambda_t \geq \lambda$ holds.
\begin{corollary}[Worst-case stability]\label{cor:worst-case}
    Given $\lambda > 0$, let $$M:=\max\big\{2/\nu,8\left(\Lambda+\varepsilon\right)/(\nu^2\mu_0^2)\big\},$$ and choose $\alpha = \frac{1}{Bk}$ and $\eta = \xi\alpha$ for $\xi \in [M, k\lambda/(2\Lambda)]$. Then,
    \begin{align*}
        \bE\|\Delta_k\|_{\bar{\BH}_{k-1}} \lesssim \sqrt{B\xi k\Lip_{\nabla\varphi,\cC}}+\sqrt{k}B\Lip_{\nabla\varphi,\cC}+ \frac{Bk}{\lambda^\frac{1}{2}}\Big(\Lip_{\nabla\varphi,\cC}+\frac{1}{n}\Big)+\frac{B\xi}{n\lambda^\frac{1}{2}}.
    \end{align*}
\end{corollary}

\begin{remark}[Optimization vs. generalization] \normalfont
Note that ($\eta,\lambda,\alpha$) in Corollary \ref{cor:worst-case} satisfy the conditions of Theorems \ref{thm:optimization} and \ref{thm:stability} for a constant $\lambda > 0$, and yields \begin{align*}
    \bE\|\Delta_k\|_{\bar{\BH}_{k-1}} \lesssim \frac{\sqrt{Bk}}{m^{\frac{1}{4}}}
\;+\;
\frac{Bk}{\lambda^{\frac{1}{2}}}\Big(\frac{1}{n}+\frac{1}{m^{\frac{1}{2}}}\Big)+ \frac{\sqrt{k}B}{\sqrt{m}}
\end{align*}
In this case, a larger damping factor $\lambda > 0$ implies better conditioning for $\bm{H}_k$ and improves the generalization bounds for \textsc{SGN}. On the other hand, the optimization error increases with $\lambda$ since $\gamma \propto \frac{\lambda}{\alpha}=\Theta(k)$ in Theorem \ref{thm:optimization}. Thus, larger $\lambda$ can hinder optimization with a non-vanishing error term of order $\cO(\lambda)$, revealing a fundamental trade-off governed by $\lambda$. This illustrates the regularization effect of damping in \textsc{SGN}.
\end{remark}
\begin{remark}[Benefit of overparameterization]\normalfont
    Theorems \ref{thm:optimization} and \ref{thm:stability} indicate that both optimization and stability bounds for SGN improve with $m$, demonstrating the benefits of overparameterization. 
\end{remark}
\begin{remark}[Benefit of cumulative preconditioning.]\normalfont
    An important factor that helps generalization is the cumulative Gauss-Newton term in the preconditioner. In benign cases where $\alpha\sum_{t=0}^k\bm{J}_t^\top(w_t)\bm{J}_t(w_t)$ is well-conditioned, favorable optimization and stability bounds can be simultaneously achieved. 
    
\end{remark}

The following assumption is akin to persistence of excitation in adaptive control and stochastic regression \citep{willems2005note, sayedana2022consistency, anderson1979strong}.
\begin{condition}[Persistence of excitation]\normalfont
    In an event $E$ in the $\sigma$-algebra $\sigma(I_t:t\in\bN)$, we have $\lambda_{\min}\Big(\sum_{s\leq t}\bm{J}_s^\top(w_s)\bm{J}_s(w_s)\Big)\geq CB(t+1)^q$ for some $C,q\in\bR_{++}\mbox{ and }k_0\in\bN$ for all $t\geq k_0$.
    \label{condition:pe}
\end{condition}
We have $\lambda_{\min}(\bm{H}_t)\geq \alpha C B(t+1)^q+\lambda$ for all $t \in\{k_0,\ldots,k\}$ in the event $E$. Using this in Theorem \ref{thm:stability}, we obtain the following stability bound.
\begin{proposition}[Stability under PE]\label{prop:pe}
In an event $\cE$ in which Condition \ref{condition:pe} holds for some $q > 0$,
\begin{align*}
    \bE[\|\Delta_k\|_{\bar{\bm{H}}_{k-1}}&\mathbbm{1}_{E}] \lesssim k\Big(\sqrt{\alpha}B^\frac{3}{2}\Lip_{\nabla\varphi,\cC}+\sqrt{\eta B\Lip_{\nabla\varphi,\cC}}\Big)+\sqrt{\alpha}B^\frac{3}{2}\Big(\frac{1}{n}+\Lip_{\nabla\varphi,\cC}\Big)k^{2-\frac{q}{2}}+ k\sqrt{\eta}\frac{B^2}{n}.
\end{align*}
    Let $M:=\max\left\{\frac{2}{\nu},\frac{8\left(\Lambda+\varepsilon\right)}{\nu^2\mu_0^2}\right\}$. Then, $\alpha = \frac{1}{Bk}$, $\eta = \frac{\xi}{k}$ and $\lambda = \frac{\gamma\Lip_{\varphi,\cC}^2}{k}$ for $\xi \geq M$ and $\gamma \geq \frac{2\Lambda\xi}{\Lip_{\varphi,\cC}^2}$ satisfy all conditions of Theorems \ref{thm:optimization} and \ref{thm:stability}, and yield
    \begin{align*}
        \bE[\|\Delta_k\|_{\bar{\bm{H}}_{k-1}}\mathbbm{1}_{E}]&\lesssim \sqrt{k}\Big(\frac{B}{\sqrt{m}}+\frac{\sqrt{B}}{m^\frac{1}{4}}\Big)+ k^\frac{3-q}{2} B \Big(\frac{1}{n}+\frac{1}{\sqrt{m}}\Big)+\frac{B\sqrt{k}}{n}.
    \end{align*}
\end{proposition}
Notably, under the persistence of excitation condition, \textsc{SGN} achieves $\widetilde{\cO}(1/k)$ convergence rate (Theorem \ref{thm:optimization} with the $\xi$ and $\gamma$ choices specified in Proposition \ref{prop:pe}) and $\cO\left(k^\frac{3-q}{2}(n^{-1}+m^{-\frac{1}{2}})+\sqrt{k}(n^{-1}+m^{-\frac{1}{4}})\right)$ stability. Particularly, linear growth $q = 1$ implies a stability bound of $\cO(k/n + \sqrt{k}/m^{1/4})$.

\subsection{Stability Analysis of SGN}\label{subsec:stability-analysis}
In this subsection, we provide a proof sketch for Theorem \ref{thm:stability}, which is highly illuminating to showcase the impact of (i) curvature, (ii) preconditioner mismatch, and (iii) metric mismatch on the stability of \textsc{SGN}.
\begin{proof}[Proof of Theorem \ref{thm:stability}]
    Let $\pi_k:=\pi_{\cC}^{\BH_k}$, $\pi_k':=\pi_{\cC}^{\BH_k'}$ and $\bar{\pi}_k:=\pi_{\cC}^{\bar{\BH}_k}$ be the projection operators, and define
    \begin{align*}
        u_k&=w_k-\eta\BH_k^{-1}\Psi_k(w_k)\\u_k'&=w_k'-\eta[\BH_k']^{-1}\Psi_k'(w_k').
    \end{align*}
    Then, we have the following error decomposition:
    \begin{align*}
        \|\Delta_{k+1}\|_{\bar{\BH}_k} \leq \|\underbrace{\bar{\pi}_k(u_k-u_k')}_{(A_k)}\|_{\bar{\BH}_k}+\|\underbrace{(\pi_k-\bar{\pi}_k)u_k}_{(B_k)}\|_{\bar{\BH}_k}+\|\underbrace{(\bar{\pi}_k-\pi_k')u_k'}_{(B_k')}\|_{\bar{\BH}_k}.
    \end{align*}
    In this inequality, $(B_k)$ and $(B_k')$ are error terms due to the metric (or projection) mismatch. We show that the critical term $(A_k)$ will yield $\|\Delta_k\|_{\bar{\BH}_{k-1}}$ plus controllable error terms with $n$ and $m$. 
    \paragraph{\textbf{\underline{Bounding $(A_{k,1})$.}}} $\cC$ is a compact and convex, and $\bar{\BH}_k$ is positive definite, thus $\bar{\pi}_k$ is non-expansive \citep{nesterov2018lectures, brezis2011functional}:
    \begin{equation*}
        \|\bar{\pi}_k(u_k-u_k')\|_{\bar{\BH}_k} \leq \|u_k-u_k'\|_{\bar{\BH}_k}.
    \end{equation*}
    Let $T_k(w) := w-\eta\bar{\BH}_k^{-1}\Psi_k(w).$ Then, we further decompose $(A_k)$ as follows:
    \begin{align*}
        \|u_k-u_k'\|_{\bbh_k} &\leq \|\underbrace{T_k(w_k)-T_k(w_k')}_{(A_{k,1})}\|_{\bbh_k} + \eta\|\underbrace{\bbh_k^{-1}(\Psi_k(w_k')-\Psi_k'(w_k'))}_{(A_{k,2})}\|_{\bbh_k}\\&+ \eta\|\underbrace{(\bbh_k^{-1}-\bm{H}_k^{-1})\Psi_k(w_k)}_{(A_{k,3})}\|_{\bbh_k}+\eta\|\underbrace{\big((\bm{H}_k')^{-1}-\bbh_k^{-1}\big)\Psi'_k(w_k')}_{(A'_{k,3})}\|_{\bbh_k}.
    \end{align*}
    In this decomposition, $(A_{k,3})$ and $(A_{k,3}')$ correspond to the \emph{preconditioner mismatch} terms. 
    \begin{lemma}[Approximate non-expansivity of $T_k$]
    For any $u,v\in\cC$, if $(\eta,\lambda)$ satisfy \eqref{eqn:hyperpars},
    then we have the (approximate) non-expansivity
    \begin{align*}
        \|T_k(u)-T_k(v)&\|_{\bar{\bm{H}}_k} \leq \|u-v\|_{\bar{\bm{H}}_{k-1}} + 2r_{\cC}\sqrt{\eta\varepsilon}+2\sqrt{\alpha}\Lip_{\nabla\varphi, \cC} B r_{\cC}\Big(\frac{r_{\cC}}{2}+\frac{\sqrt{B}\Lip_{\ell,\cC}}{\lambda_0\nu}\Big)
    \end{align*}
    almost surely.
        \label{lemma:contractivity}
    \end{lemma}
    A key part to prove Lemma \ref{lemma:contractivity} is to show an approximate co-coercivity in the overparameterized regime.

    \paragraph{\textbf{\underline{Bounding $(A_{k,2})$.}}} This term corresponds to \emph{gradient mismatch}. Since $\{j^\star \notin  I_k\}\subset\{\Psi_k(\cdot)=\Psi_k'(\cdot)\}$, we obtain
    \begin{align}
        \nonumber \|\bbh_k^{-1}(\Psi_k(w_k')-\Psi_k'(w_k'))\|_{\bbh_k}\nonumber &\leq 2\sup_{w\in\cC}\|\Psi_k(w)\|_{\bbh_k^{-1}}\mathbbm{1}_{\{j^\star\in I_k\}}\\
        &\label{eqn:hk-bound-1}\leq \frac{2B}{\lambda_{\min}^{1/2}(\bbh_k)}\Lip_{\ell,\cC}\Lip_{\varphi,\cC}\mathbbm{1}_{\{j^\star\in I_k\}}.
    \end{align}

    \textbf{\underline{Bounding $(A_{k,3})$ and $(A_{k,3}')$.}} To bound these error terms that stem from the preconditioner mismatch, we use
    \begin{align}\label{eqn:hk-bound-2}
        \|(\bbh_k^{-1}-\bm{H}_k^{-1})\Psi_k(w_k)\|_{\bbh_k}\leq \|\bar{\bm{H}}_k^{-1/2}(\bar{\bm{H}}_k-\bm{H}_k)\bm{H}_k^{-1}\|_2\|\Psi_k(w_k)\|_2.
    \end{align}
    We have
        $$\lambda_{\max}(\bbh_k) \leq \alpha B (k+1)\Lip_{\varphi,\cC}^2+\lambda,$$
       $\sup_{w\in\cC}\|\Psi_k(w)\|_2\leq B\Lip_{\ell,\cC}\Lip_{\varphi,\cC},$
    and $$\|\bm{H}_k'-\bbh_k\|_2=\|\bm{H}_k-\bbh_k\|_2=\frac{1}{2}\|\bm{H}_k-\bm{H}_k'\|_2.$$ The following lemma upper bounds the error term $\bE\|\bm{H}_k-\bm{H}_k'\|_2$.
    \begin{lemma}[Stability of $\bm{H}_k$]
        For any $k\in\bN$,
        \begin{align*}
            \bE[\|\bm{H}_k-\bm{H}_k'\|_2] &\leq 2B\alpha(k+1) r_{\cC}\Lip_{\varphi,\cC}\Lip_{\nabla\varphi,\cC}+\alpha\Lip_{\varphi,\cC}^2\frac{(k+1)B}{n}.
        \end{align*}
    \end{lemma}
    \paragraph{\textbf{\underline{Bounding $(B_k)$ and $(B_k')$.}}} Finally, we bound the error due to metric (i.e., projection) mismatch.
    \begin{lemma}[Metric mismatch]
        For any $k\in\bN$, we obtain
        \begin{align}
            \|(\pi_k-\bar{\pi}_k)u_k\|_{\bbh_k} \nonumber &\leq \|\bbh_k^{-1/2}(\bm{H}_k-\bbh_k)\|_2\sup_{w\in\cC}\|w-u_k\|_2\\
            \label{eqn:hk-bound-3}&\leq \frac{\|\bm{H}_k-\bm{H}_k'\|_2}{2\lambda^{1/2}_{\min}(\bbh_k)}\sup_{w\in\cC}\|w-u_k\|_2
        \end{align}
        Also, for any $w\in\cC$,
            $\|w-u_k\|_2 \leq  r_{\cC}+\frac{\eta B}{\lambda}\Lip_{\ell,\cC}\Lip_{\varphi,\cC}.$
    \end{lemma}
    The proof of Theorem \ref{thm:stability} follows from substituting the inequalities to establish an upper bound for $\bE\|\Delta_{k+1}\|_{\bbh_k}-\bE\|\Delta_k\|_{\bbh_{k-1}}$, telescoping sum over $k$, and finally noting that $\Delta_0=0$ since $w_0=w_0'$.
\end{proof}

\section{Conclusions}
In this work, we analyzed SGN for deep neural networks in the near-initialization regime. We established optimization and (algorithm-dependent) generalization bounds for SGN, with explicit dependencies on key factors such as damping, overparameterization, batch size, training duration, and the spectrum of the preconditioner. Our analysis demonstrates the robustness of SGN in regimes with ill-conditioned kernels and large loss curvature in deep learning. To the best of our knowledge, this is the first work to establish stability-based generalization bounds for Gauss-Newton in deep learning. Interesting directions for future research include extension of the analysis framework developed in this paper to more general preconditioned (e.g., adaptive gradient or Hessian-based) methods in deep learning, and also extension to other neural architectures.
\newpage

\bibliography{refs}

\appendix
\section{Computational Efficiency of Stochastic Gauss-Newton}
A key idea to make the \textsc{SGN} iterations computationally efficient is extended Kalman filter updates, introduced in \citep{bertsekas1996incremental}.
At iteration $k\in\bN$, we set $\theta_k^{(0)}:=w_k$ and $\bm{Z}_k^{(0)}:=\bm{H}_k$, and $g_k^{(j)}:= \nabla_w \varphi(\bm{x}_{I_k^{(j)}};w_k)$. Then, for $j=1,\ldots,B$, the following update is performed:
\begin{align*}
    \bm{Z}_k^{(j)} &= \bm{Z}_k^{(j-1)} + \alpha\cdot  g_k^{(j)}[g_k^{(j)}]^\top,\\
    \theta_k^{(j)} &= \theta_k^{(j-1)} - M_k^{(j)}\cdot [\bm{Z}_k^{(j)}]^{-1} g_k^{(j)},
\end{align*}
where
$$M_k^{(j)}:=\alpha [g_k^{(j)}]^\top (\theta_k^{(j)}-w_k)-\eta\ell'(\varphi(\bm{x}_{I_k^{(j)}};w_k);\bm{z}_{I_k^{(j)}}).$$
We set $\bm{H}_{k+1}=\bm{Z}_k^{(B)}$ and $w_{k+1}=\pi_{\cC}^{\bm{H}_k}(\theta_k^{(B)})$. For computational efficiency, Sherman-Morrison-Woodbury formula \citep{horn2012matrix} can be used. To that end, let $\tilde{g}_k^{(j)}:=[\bm{Z}_k^{(j-1)}]^{-1}g_k^{(j)}$. Then, for each $j=1,2,\ldots,B$,
\begin{equation*}
    [\bm{Z}_k^{(j)}]^{-1}g_k^{(j)} = \tilde{g}_k^{(j)}\Big(1 - \frac{[\tilde{g}_k^{(j)}]^\top}{[g_k^{(j)}]^\top\tilde{g}_k^{(j)}}\Big).
\end{equation*}
The computational complexity of the overall operation is $\cO(Bp^2)$, whereas the complexity of the original Gauss-Newton update is $\cO(p^3)$. The projection step can be performed in $\cO(pr+r^3)$ for (i) exactly via Sherman-Morrison-Woodbury with $r \leq kB$, and (ii) approximately for tunable $r$ via randomized sketching \citep{pilanci2017newton, xu2016sub}.

\section{Proofs for Section \ref{sec:opt}}\label{app:opt}
The following elementary results will be useful in the proofs.
\begin{lemma}\label{lemma:norm}
    For any $\bm{x}\in\cB_2(0,1)$ and $w,w'\in\cC$, we have
    \begin{equation*}
        |\varphi(\bm{x};w') - \varphi(\bm{x};w)- \nabla^\top\varphi(\bm{x};w)\big(w'-w\big)|\leq \frac{1}{2}\Lip_{\nabla\varphi,\cC}r_{\cC}^2.
    \end{equation*}
    For any $k\in\bN$, let 
    \begin{equation}
        \Bf_k(w) := [\varphi(\bm{x}_j;w)]_{j\in I_k} \in \bR^B
    \end{equation}
    be output vector of the neural network with parameter $w$ evaluated on the minibatch $I_k$, and $$\bm{\epsilon}_k := \Bf_k(\bar{w})-\Bf_k(w_k)-\bm{J}_k(w_k)\big(\bar{w}-w_k\big).$$
    Then, we have
        $\|\Be_k\|_2 \leq \frac{\sqrt{B}\Lip_{\nabla\varphi,\cC}r_{\cC}^2}{2}\quad\mbox{\textit{a.s.}}$
\end{lemma}
\begin{proof}[Proof of Lemma 1]
    Let $\mathrm{D}_{\bm{x}^{(h-1)}}\bm{x}^{(h)}=\frac{1}{\sqrt{m}}\bm{\Delta}_hW^{(h)}$, where $$\bm{\Delta}_h := \mathrm{diag}\Big(\vec{\sigma}'\Big(W^{(h)}\bm{x}^{(h-1)}(\bm{W})\Big)\Big).$$ Then, $\|\mathrm{D}_{\bm{x}^{(h-1)}}\bm{x}^{(h)}\|_2\leq \sigma_1\kappa_{\cC}.$ Let 
    \begin{align*}
        u_H &:= \nabla_{\bm{x}^{(H)}(\bm{W})}\varphi(\bm{x};w)=c,\\
        u_{h-1} &:= \mathrm{D}_{\bm{x}^{(h-1)}}^\top\bm{x}^{(h)}u_h,~h=1,2,\ldots,H.
    \end{align*}
    Then, for any $h\in[H],$ $$\|u_h\|_2 \leq (\sigma_1\kappa_{\cC})^{H-h}\zeta_{\cC}.$$ We have
    \begin{align*}
        \nabla_c\varphi(\bm{x};w) &= \bm{x}^{(h)}(\bm{W}),\\
        \nabla_{W^{(h)}}\varphi(\bm{x};w)&= \frac{1}{\sqrt{m}}\bm{\Delta}_hu_h\big(\bm{x}^{(h-1)}(\bm{W})\big)^\top.
    \end{align*}
    Using the bound for $\|u_h\|_2$ established above, and $\|\bm{x}^{(h)}(\bm{W})\|_2 \leq \sigma_0$, and noting that $\|\nabla_{\bm{W}}\varphi(\bm{x};w)\|_2 \leq \zeta_{\cC}\frac{\sigma_0\sigma_1}{\sqrt{m}}\sqrt{H}(\sigma_1\kappa_{\cC})^{H-1}$, we obtain
    \begin{equation}
        \|\nabla_w\varphi(\bm{x};w)\|_2 \leq \sigma_0 + \zeta_{\cC}\frac{\sigma_0\sigma_1}{\sqrt{m}}\sqrt{H}(\sigma_1\kappa_{\cC})^{H-1}.
    \end{equation}

    For the second part, let $\delta w=(\delta W^{(1)},\ldots,\delta W^{(H)},\delta c)$ with $\|\delta w\|_2=1$, and let
$\delta(\cdot)$ denote the Fr\'echet derivative in direction $\delta w$. Then
\[
\delta\!\big(\nabla_c\varphi\big)=\delta \mathbf{x}^{(H)},\qquad
\delta\!\big(\nabla_{W^{(h)}}\varphi\big)
=\frac{1}{\sqrt{m}}\Big[(\delta\Delta_h)u_h\,\mathbf{x}^{(h-1)\top}
+\Delta_h\,\delta u_h\,\mathbf{x}^{(h-1)\top}
+\Delta_h u_h\,\delta\mathbf{x}^{(h-1)\top}\Big].
\]
Using $\|\delta\Delta_h\|_2\le \sigma_2\|\delta(W^{(h)}\mathbf{x}^{(h-1)})\|_2$ and
\[
\|\delta\mathbf{x}^{(t)}\|_2\le 2\sigma_0\sigma_1\frac{H}{\sqrt{m}}\kappa_{\mathcal{C}}^{\,t-1},
\qquad
\|\delta u_h\|_2\le (\sigma_2\sigma_0+\sigma_1^2\zeta_{\mathcal{C}})\frac{H}{\sqrt{m}}\kappa_{\mathcal{C}}^{\,H-h},
\]
together with $\|\mathbf{x}^{(t)}\|_2\le\sigma_0$, $\|u_h\|_2\le(\sigma_1\kappa_{\mathcal{C}})^{H-h}\zeta_{\mathcal{C}}$, and
$\|\Delta_h\|_2\le\sigma_1$, each bracketed term above is bounded by $(\sigma_2\sigma_0+\sigma_1^2\zeta_{\mathcal{C}})\,\frac{H}{m}\,\kappa_{\mathcal{C}}^{\,H-1},$ which yields
\[
\|\nabla_w^2\varphi(\mathbf{x};w)\|_2
\;\le\; 8\,(\sigma_2\sigma_0+\sigma_1^{2}\zeta_{\mathcal{C}})\,
\left[\frac{H\,\kappa_{\mathcal{C}}^{\,H-1}}{\sqrt{m}}
+\left(\frac{H\,\kappa_{\mathcal{C}}^{\,H-1}}{\sqrt{m}}\right)^{\!2}\right].
\]
\end{proof}

\begin{proof}[Proof of Theorem 1]
The key ideas behind analyzing this variant of Gauss-Newton is to interpret it as a combination of a variable-metric method akin to adaptive gradient methods \citep{duchi2018introductory}, which have path-dependent and time-variant Bregman divergences. In essence, we exploit this insight in quantifying the impact of time-variance of the Bregman divergence ($(a_k)$ below) and bounding the gradient norm ($(b_{k,2})$ below). The most significant component of the Lyapunov drift is the negative drift term ($(b_{k,1})$ below), which we handle by using the specific properties of SGN.

    First, let $$\mathscr{L}_k(w) := \|w-\bar{w}\|_{\bm{H}_k}^2,\quad w\in\cC$$ be the Lyapunov function. We apply the following decomposition \citep{duchi2018introductory}:
    \begin{align*}
        \sL_{k+1}(w_{k+1}) &= \sL_k(w_{k+1})+\|w_k-\bar{w}\|_{\bm{H}_{k+1}-\bm{H}_k}^2\\
        &\leq \|u_k-\bar{w}\|_{\BH_k}+\|w_k-\bar{w}\|_{\bm{H}_{k+1}-\bm{H}_k}^2\\
        &= \sL_k(w_k)+2\eta\underbrace{(\bar{w}-w_k)^\top\Psi_k(w_k)}_{=:(b_{k,1})} + \eta^2\underbrace{\|\Psi_k(w_k)\|_{\BH_k^{-1}}^2}_{=:(b_{k,2})} + \underbrace{\|w_k-\bar{w}\|_{\bm{H}_{k+1}-\bm{H}_k}^2}_{=:(a_k)}.
    \end{align*}
     Consider the filtration $\sF_k:=\sigma(I_0,I_1,\ldots,I_k),~k\in\bN$, which is the history of the optimization path up to time $k$. As elementary properties, notice that $w_k$ is $\sF_{k-1}$-measurable, and $$\bE[\Phi_k(w_k)\mid\sF_{k-1}]=B\cdot \er(\bv(w_k))\Rightarrow \bE[\Phi_k(w_k)-\Phi_k(\bar{w})\mid \sF_{k-1}]=B\Big(\er(\bv(w_k))-\er(\bar{w})\Big)$$ for each $k$.
    \paragraph{\textbf{\underline{Bounding $(b_{k,1})$.}}} Recall that $\Psi_k(w) = \bm{J}_k^\top(w)\bm{G}_k(w)$ and $\Be_k=\Bf_k(\bar{w})-\Bf_k(w_k)-\bm{J}_k(w_k)\big(\bar{w}-w_k)$. Then,
    \begin{align*}
        (b_{k,1}) &= \bm{G}_k^\top(w_k)\bm{J}_k(w_k)\big(\bar{w}-w_k\big)\\
        &= \bm{G}_k^\top(w_k)\Big(\Bf_k(\bar{w})-\Bf_k(w_k)-\Be_k\Big)\\
        &\leq \bm{G}_k^\top(w_k)\big(\Bf_k(\bar{w})-\Bf_k(w_k)\big) + \frac{B}{2}\cdot\Lip_{\ell,\cC}\cdot\Lip_{\nabla\varphi,\cC}\cdot r_{\cC}^2,
    \end{align*}
    since $\|\bm{G}_k(w)\|_2\leq \sqrt{B}\Lip_{\ell,\cC}$ and $\|\Be_k\|_2 \leq \frac{1}{2}\Lipdf r_{\cC}^2$ by Lemma \ref{lemma:norm}. Using the $\nu$-strong convexity of $\ell$,
    \begin{equation*}
        \bm{G}_k^\top(w_k)\big(\Bf_k(\bar{w})-\Bf_k(w_k)\big) \leq \Phi_k(\bar{w})-\Phi_k(w_k) - \frac{\nu}{2}\|\Bf_k(w_k)-\Bf_k(\bar{w})\|_2^2.
    \end{equation*}
    Let $V_k := \cE(w_k)$ be the optimality gap. Then, we have
    \begin{equation*}
        \bE[(b_{k,1})\mid\sF_{k-1}] \leq -B\cdot V_k -\frac{\nu}{2}\bE[\|\Bf_k(w_k)-\Bf_k(\bar{w})\|_2^2\mid \sF_{k-1}] + \frac{B}{2}\cdot\Lip_{\ell,\cC}\cdot\Lip_{\nabla\varphi,\cC}\cdot r_{\cC}^2.
    \end{equation*}
    By the law of iterated expectation, we obtain
    \begin{equation}\label{eqn:b-k1}
        \bE[(b_{k,1})] \leq -B\cdot\bE[V_k]-\frac{\nu}{2}\bE[\|\Bf_k(w_k)-\Bf_k(\bar{w})\|_2] + \frac{B}{2}\cdot\Lip_{\ell,\cC}\cdot\Lip_{\nabla\varphi,\cC}\cdot r_{\cC}^2.
    \end{equation}

    \paragraph{\textbf{\underline{Bounding $(b_{k,2})$.}}} This part corresponds to the squared norm of the preconditioned gradient update with respect to the $\bm{H}_k$-norm. The initial term is bounded by using Sherman-Morrison-Woodbury matrix identity, and the proceeding terms are bounded by exploiting trace inequalities (building on \cite{hazan2007logarithmic}) and the incremental nature of $\bm{H}_k$.

    First, note that 
    \begin{align*}
        \|\bm{J}_k^\top(w_k)\BG_k(w_k)\|_{\bm{H}_{k}^{-1}}^2 &= \BG_k^\top (w_k)\bm{J}_k(w_k)\bm{H}_{k}^{-1}\bm{J}_k^\top(w_k)\BG_k(w_k)\\
        &\leq B\cdot \Lipl^2\cdot \rr{tr}\Big(\bm{J}_k^\top(W_k)\bm{J}_k(w_k)\bm{H}_k^{-1}\Big),
    \end{align*}
    where the last inequality follows from the circular shift invariance of $\rr{tr}$. By Klein's inequality (Lemma 1 in \cite{landford1967mean}, also see \cite{hazan2007logarithmic}), we have 
    \begin{equation*}
        \rr{tr}\log\bm{H}_{k+1}-\alpha\cdot \rr{tr}\Big(\bm{J}_{k+1}^\top(w_{k+1})\bm{J}_{k+1}(w_{k+1})\bm{H}_{k+1}^{-1})\Big)-\rr{tr}\log\bm{H}_k \geq 0.
    \end{equation*}
    Hence, we have
    \begin{align*}
        \sum_{k=0}^{t-1}(b_{k,2}) &\leq B\cdot \Lipl^2\Bigg(\|\bm{J}_0(w_0)\bm{H}_0^{-1}\bm{J}_0^\top(w_0)\|_2+\sum_{k=1}^{t-1}\rr{tr}\Big(\bm{J}_k^\top(w_k)\bm{J}_k(w_k)\bm{H}_k^{-1}\Big)\Bigg)\\
        &\leq B\cdot \Lipl^2\Big(\|\bm{J}_0(w_0)\bm{H}_0^{-1}\bm{J}_0^\top(w_0)\|_2+\frac{1}{\alpha}\big(\rr{tr}\log \bm{H}_{t}-\rr{tr}\log\bm{H}_0\big)\Big).
    \end{align*}
    In order to bound $\|\bm{J}_0(w_0)\bm{H}_0^{-1}\bm{J}_0^\top(w_0)\|_2$, we use the Sherman-Morrison-Woodbury matrix identity \citep{horn2012matrix, cayci2024riemannian}:
    \begin{equation*}
        \bm{J}_0(w_0)\bm{H}_0^{-1}\bm{J}_0^\top(w_0) = \frac{1}{\lambda}\Big(\bm{K}_0-\frac{\alpha}{\lambda}\bm{K}_0[\bm{I}+\alpha\lambda^{-1}\bm{K}_0]^{-1}\bm{K}_0\Big),
    \end{equation*}
    where $\bm{K}_0:=\bm{J}_0(w_0)\bm{J}_0^\top(w_0)\in\bR^{n\times n}$. If $(\mu^2,u)\in\bR^+\times\bR^n$ is an eigenpair of $\bm{K}_0$, then $\Big(\frac{\mu^2}{\lambda+\alpha\mu^2}, u\Big)$ is an eigenpair of $\bm{J}_0(w_0)\bm{H}_0\bm{J}_0^\top(w_0)$. Note that $s\mapsto \frac{s}{\lambda+\alpha s}$ is an increasing function, thus
    \begin{align*}
        \|\bm{J}_0(w_0)\bm{H}_0\bm{J}_0^\top(w_0)\|_2 &\leq \frac{\|\bm{K}_0\|_2}{\lambda+\alpha\|\bm{K}_0\|_2}\\
        &\leq \frac{B\cdot \Lipf^2}{\lambda + \alpha\cdot B\cdot \Lipf^2}
    \end{align*}
    where the second line follows from
    \begin{equation*}
        \sup_{w\in\cC}\|\bm{J}_k(w)\|_2 \leq \sqrt{B}\Lipf\quad{a.s.\mbox{ for all }k\in\bN}.
    \end{equation*}
    Also, note that $$\rr{tr}\log\bm{H}_t = \log\det\bm{H}_t,~t\in\bN.$$
    Putting everything together, we obtain:
    \begin{align*}
        \sum_{k=0}^{t-1}(b_{k,2}) \leq B\cdot \Lipl^2\Big(\frac{B\cdot \Lipf^2}{\lambda + \alpha\cdot B\cdot \Lipf^2} + \frac{1}{\alpha}\cdot\log\frac{\det\bm{H}_t}{\det \bm{H}_0}\Big).
    \end{align*}
    For any $\gamma > 0$, let $\lambda := \gamma\alpha B \Lipf^2$. Substituting this into the above inequality and taking expectation yields
    \begin{equation}\label{eqn:b-k2}
        \sum_{k=0}^{t-1}\bE\left[(b_{k,2})\right] \leq \frac{B\Lipl^2}{\alpha}\Big(\frac{1}{1+\gamma}+\bE\log\frac{\det\bm{H}_t}{\det\bm{H}_0}\Big).
    \end{equation}

    \paragraph{\textbf{\underline{Bounding $(a_k)$.}}} Note that the Bregman divergence in SGN updates are time-variant, and this error term quantifies the impact of this metric change at $k$-th iteration. First, note that $$\|w_{k+1}-\bar{w}\|_{\bm{H}_{k+1}-\bm{H}_k}^2 = \alpha \|\bm{J}_{k+1}(w_{k+1})(w_{k+1}-\bar{w})\|^2.$$ By using Lemma \ref{lemma:norm} and $(x+y)^2\leq 2(x^2+y^2)$,
    \begin{align*}
        \|\bm{J}_{k+1}(w_{k+1})(w_{k+1}-\bar{w})\|^2 &= \| \bm{\epsilon}_{k+1} - \Bf_{k+1}(\bar{w})+\Bf_{k+1}(w_{k+1})\|_2^2\\
        &\leq \frac{{B}\Lip^2_{\nabla\varphi,\cC}r_{\cC}^4}{2} + 2\cdot \|\Bf_{k+1}(\bar{w})-\Bf_{k+1}(w_{k+1})\|_2^2.
    \end{align*}
    Hence, we obtain
    \begin{equation}\label{eqn:a-k}
        \sum_{k=0}^{t-1}\bE[(a_k)] \leq \alpha t B \Lipdf^2r_{\cC}^4 + 2\alpha\sum_{k=1}^t\bE\|\Bf_{k}(\bar{w})+\Bf_{k}(w_{k})\|_2^2.
    \end{equation}

    \paragraph{\underline{\textbf{Drift analysis.}}} Note that
    \begin{align*}
        \bE[\sL_t(w_t)]-\bE[\sL_0(w_0)] &= \sum_{k=0}^{t-1}\bE[\sL_{k+1}(w_k)-\sL_k(w_k)]\\
        &\leq 2\eta \sum_{k=0}^{t-1}\bE[(b_{k,1})] + \eta^2\sum_{k=0}^{t-1}\bE[(b_{k,2})] + \sum_{k=0}^{t-1}\bE[(a_k)].
    \end{align*}
    Substituting \eqref{eqn:b-k1},\eqref{eqn:b-k2} and \eqref{eqn:a-k} into the above inequality, we obtain
    \begin{align*}
        \bE[\sL_t(w_t)]-\bE[\sL_0(w_0)] 
        &\leq -2\eta B\sum_{k=0}^{t-1}\bE[V_k] +(2\alpha-\nu\eta)\sum_{k=1}^{t-1}\bE\|\Bf_{k}(\bar{w})+\Bf_{k}(w_{k})\|_2^2\\
        &+2\alpha\bE\|\bm{f}_t(w_t)-\bm{f}_t(\bar{w})\|_2^2+\eta t B \Lipl\Lipdf r_{\cC}^2\\
        &+\frac{\eta^2 B \Lipl^2}{\alpha}\Big(\frac{1}{1+\gamma}+\bE\log\frac{\det\bm{H}_t}{\det\bm{H}_0}\Big)+\frac{\alpha t B \Lipdf^2r_{\cC}^4}{2}.
    \end{align*}
    By rearranging terms and noting that $\sL_t(\cdot)>0$ and
    $$
    \|\bm{f}_t(\bar{w})-\bm{f_t}(w_t)\|_2^2=\sum_{j\in I_t}|\varphi(\bm{x}_j;w_t)-\varphi(\bm{x}_j;\bar{w})|^2 \leq B\Lipf^2r_{\cC}^2,
    $$
    we obtain
    \begin{align*}
        \frac{1}{t}\sum_{k=0}^{t-1}\bE[V_k]&\leq \frac{\bE\sL_0(w_0)}{2\eta Bt} + \frac{\alpha}{\eta t}\Lipf^2r_{\cC}^2 + \frac{\eta}{\alpha}\frac{\Lipl^2}{2t}\Big(\frac{1}{1+\gamma}+\bE\log\frac{\det\bm{H}_t}{\det\bm{H}_0}\Big)\\
        &+ \frac{1}{4}\Lipl r_{\cC}^2\Lipdf + \frac{1}{4}\cdot\frac{\alpha}{\eta}\cdot r_{\cC}^4\cdot \Lipdf^2.
    \end{align*}
    We have
    \begin{align*}
        \sL_0(w_0)=\|w_0-\bar{w}\|_{\BH_0}^2 &\leq r_{\cC}^2\Big(\alpha B \Lipf^2+\lambda\Big)\\
        &= \alpha\cdot r_{\cC}^2\cdot (\gamma+1)\cdot B\cdot \Lipf^2
    \end{align*}
    almost surely. Using these and the definition $\xi:=\frac{\eta}{\alpha}$, we finally get
    \begin{align}\label{eqn:final-opt}
        \begin{aligned}
        \frac{1}{t}\sum_{k=0}^{t-1}\bE[V_k]&\leq \frac{\gamma+2}{\xi t}\cdot \Lipf^2\cdot r_{\cC}^2 + \xi \frac{\Lipl^2}{2t}\Big(\frac{1}{1+\gamma}+\bE\log\frac{\det\bm{H}_t}{\det\bm{H}_0}\Big)\\
        &+ \frac{1}{4}\Lipl r_{\cC}^2\Lipdf + \frac{1}{4\xi}\cdot r_{\cC}^4\cdot \Lipdf^2.
        \end{aligned}
    \end{align}
    This concludes the proof of the first part.

    In order to prove the average-iterate convergence result, we first define the following linearization:
    \begin{equation*}
        \varphilin(\bm{x};w) = \varphi(\bm{x};w_0)+\nabla^\top\varphi(\bm{x};w_0)\big(w-w_0\big),\quad \bm{x}\in\cB_2(0,1),w\in\cC,
    \end{equation*}
    and use the fact that 
    \begin{equation}\label{eqn:linearization}
        \sup_{\substack{\bm{x}\in\cB_2(0,1)\\ w\in\cC}}|\varphi(\bm{x};w)-\varphilin(\bm{x};w)|\leq \frac{1}{2}\Lipdf r_{\cC}^2.
    \end{equation}
    Then, for any $w\in\cC$,
    \begin{align*}
        \er(\bm \varphi(w)) &= \frac{1}{n}\sum_{j=1}^n\ell(\varphilin(\bm{x}_j;w);\bm{z}_j)+\frac{1}{n}\sum_{j=1}^n \Big(\ell(\varphi(\bm{x}_j;w);\bm{z}_j)-\ell(\varphilin(\bm{x}_j;w);\bm{z_j})\Big)\\
        &\leq \frac{1}{n}\sum_{j=1}^n\ell(\varphilin(\bm{x}_j;w);\bm{z}_j) + \frac{1}{2}\Lipl\Lipdf r_{\cC}^2,
    \end{align*}
    and similarly
    \begin{align*}
        \frac{1}{n}\sum_{j=1}^n\ell(\varphilin(\bm{x}_j;w);\bm{z}_j) \leq \er(\bv(w)) + \frac{1}{2}\Lipl\Lipdf r_{\cC}^2.
    \end{align*}
    Now, note that $w\mapsto\varphilin(\bm{x};w)$ is an affine function, and $y'\mapsto \ell(y';\bm{z})$ is (strongly) convex, which implies that $w\mapsto \ell(\varphilin(\bm{x}_j;w);\bm{z}_j)$ is convex for each $j\in[n]$. Thus, Jensen's inequality implies
    \begin{align}
        \frac{1}{n}\sum_{j\in[n]}\ell(\varphilin(\bm{x}_j;\widehat{w}_k); \bm{z}_j) &\leq \frac{1}{k}\sum_{t=0}^{k-1}\frac{1}{n}\sum_{j\in[n]}\ell(\varphilin(\bm{x}_j;w_t); \bm{z}_j)\\
        &\leq \frac{1}{k}\sum_{t=0}^{k-1}\er(\bv(w_t)) + \frac{1}{2}\Lipl\Lipdf r_{\cC}^2.
    \end{align}
    Finally, using, we obtain
    \begin{equation*}
        \er(\bm\varphi(\widehat{w}_k)) \leq \frac{1}{n}\sum_{j\in[n]}\ell(\varphilin(\bm{x}_j;\widehat{w}_k);\bm{z}) + \frac{1}{2}\Lipl\Lipdf r_{\cC}^2 \leq \frac{1}{k}\sum_{t=0}^{k-1}\er(\bm\varphi(w_t)) + \Lipl\Lipdf r_{\cC}^2.
    \end{equation*}
    Taking expectation yields
    \begin{align}
        \nonumber \bE[\er(\bm\varphi(\widehat{w}_k))]-\er(\bm\varphi(\bar{w})) &\leq \underbrace{\frac{1}{k}\sum_{t=0}^{k-1}\bE[V_t]}_{(\spadesuit)} + \Lipl\Lipdf r_{\cC}^2\\
        &\begin{aligned}\label{eqn:avg-iterate-error}&\leq \frac{\gamma+2}{\xi k}\cdot \Lipf^2\cdot r_{\cC}^2 + \xi \frac{\Lipl^2}{2k}\Big(\frac{1}{1+\gamma}+\bE\log\frac{\det\bm{H}_k}{\det\bm{H}_0}\Big)\\
        &\qquad\qquad\quad\quad + \frac{5}{4}\Lipl r_{\cC}^2\Lipdf + \frac{1}{4\xi}\cdot r_{\cC}^4\cdot \Lipdf^2,
        \end{aligned}
    \end{align}
    which concludes the proof.
\end{proof}

\begin{proof}[Proof of Proposition 1]

    Let $$\bm{S}_k := \sum_{t=0}^k\bm{J}_t^\top(w_t)\bm{J}_t(w_t),\quad k\in\bN.$$ Then, we have $$\log\frac{\det\bm{H}_k}{\det\bm{H}_0}=\log\det\Big(\bm{I}+\frac{\alpha}{\lambda}\bm{S}_k\Big).$$ Denote the non-zero eigenvalues of $\bm{S}_k$ by $\mu_{i},~i=1,2,\ldots,r'_k$ where $r'_k=\rr{rank}(\bm{S}_k)\leq p$ is the intrinsic rank. Then, using
    \begin{align*}
        \det\big(\bm{I}+\frac{\alpha}{\lambda}\bm{S}_k\big)=\sum_{i=1}^{r'_k}\log(1+\frac{\alpha}{\lambda}\mu_i)&=r'_k\cdot\frac{1}{r'_k}\sum_{i=1}^{r'_k}\log(1+\frac{\alpha}{\lambda}\mu_i)\\
        &\leq r'_k\log\Big(1+\frac{\alpha/\lambda}{r'_k}\sum_{i=1}^{r'_k}\mu_i\Big),
    \end{align*}
    where the last inequality stems from Jensen's inequality and the concavity of $x\mapsto\log(1+x)$. Note that $\sum_{i=1}^{r'_k}\mu_i=\rr{tr}(\bm{S}_k)$ and furthermore $\rr{tr}(\bm{S}_k) \leq (k+1)B\Lipf^2$. Thus,
    \begin{equation}
        \log\frac{\det\bm{H}_k}{\det\bm{H}_0}\leq r'_k\log\Big(1+\frac{\alpha}{\lambda r'_k}(k+1)B\Lipf^2\Big).
    \end{equation}
    Note that $r'_k\leq p$, which gives the first inequality. For the second inequality, note that 
    \begin{equation}
        \bE\log\frac{\det\bm{H}_k}{\det\bm{H}_0}\leq \log\det\Big(\bm{I}+\frac{\alpha}{\lambda}B(k+1)\bar{\Sigma}_{k+1}\Big).
    \end{equation}
    Then, 
    \begin{align*}
        \log\det\Big(\bm{I}+\frac{\alpha}{\lambda}B(k+1)\bar{\Sigma}_{k}\Big) &\leq \sum_{i=1}^{\bar{r}_{k}}\log(1+\frac{\alpha}{\lambda}B(k+1)\bar{\mu}_i)\\
        &\leq \sum_{i}\Big(\log\Big(1+\frac{\alpha}{\lambda}B(k+1)\Big)+\log(1+\bar{\mu}_i)\Big)\\&\leq \bar{r}_k\log\Big(1+\frac{\alpha}{\lambda} B(k+1)\Big)+\bar{r}_k\log(1+\rr{tr}(\bar{\Sigma}_k)/\bar{r}_k)
    \end{align*}
    The proof follows by noting that $\lambda = \gamma \alpha B\Lipf^2$:
    \begin{equation}
        \bE[\log\frac{\det\bm{H}_k}{\det\bm{H}_0}]\leq \bar{r}_k\log\Big(1+\frac{k+1}{\gamma\Lipf^2}\Big)+\bar{r}_k\log(1+\rr{tr}(\bar{\Sigma_k}/\bar{r}_k).
    \end{equation}
\end{proof}

\begin{proof}[Proof of Corollary 1]
    The idea here is to construct a parameter set $\cC$ such that there exists $w_f^\star\in\cC$ that approximates $\{f^\star(\bm{x}_j):j\in[n]\}$ well. To that end, first let
    \begin{equation}
        c_i^\star := c_0^{(i)}+\frac{1}{\sqrt{m}}v_c(u_0^{(i)})\quad \mbox{ and }\quad W_i^\star:=W_0^{(i)} + \frac{1}{\sqrt{m}}v_W(u_0^{(i)}),
    \end{equation}
    where $(v_c,v_W)$ is the transportation mapping. Set $w_f^\star := [(c_i^\star,W_i^\star)]_{i\in[m]}$. Note that \begin{align*}
        \|w_f^\star-w_0\|_2^2 &= \sum_{i=1}^m(\underbrace{|c_i^\star-c_0^{(i)}|^2}_{\leq \frac{1}{m}\bar{v}_c^2}+\underbrace{\|w_i^\star-w_0^{(i)}\|_2^2}_{\leq \frac{1}{m}\bar{v}_W^2})\\ &\leq \|(\bar{v}_c,\bar{v}_W)\|_2^2,
    \end{align*}
    hence we always have
    \begin{equation}
        w_f^\star \in \cB_2(w_0, \|(\bar{v}_c,\bar{v}_W)\|_2)\subset \cC.
    \end{equation}
    Next, we show that $\max_{i\in[n]}|f^\star(\bm{x}_i)-\varphi(\bm{x}_i;w_f^\star)|\lesssim \sqrt{\frac{\log(2n/\delta)}{m}}$ with high probability over the random initialization to bound $\er(\bv(w_f^\star))$. Fix $\bm{x}\in\cB_2(0,1)$ first. Then,
    \begin{align}
        \varphilin(\bm{x};w_f^\star) &:= \underbrace{\varphi(\bm{x};w_0)}_{=0}+\nabla^\top\varphi(\bm{x};w_0)(w_f^\star-w_0)\\
        &=\frac{1}{m}\sum_{i\in[m]}\Big(\sigma(\langle W_0^{(i)},\bm{x}\rangle)v_c(u_0^{(i)}) + c_0^{(i)}v_W^\top(u_0^{(i)})\bm{x}\sigma'(\langle W_0^{(i)},\bm{x}\rangle)\Big)\overset{a.s.}{\underset{m\rightarrow\infty}\longrightarrow} f^\star(\bm{x}),
    \end{align}
    by the Strong Law of Large Numbers. Now, for each $i\in[m]$, 
    \begin{equation*}
        |\sigma(\langle W_0^{(i)},\bm{x}\rangle)v_c(u_0^{(i)})| \leq \sigma_0\bar{v}_c,
    \end{equation*}
    and
    $$
        |c_0^{(i)}v_W^\top(u_0^{(i)})\bm{x}\sigma'(\langle W_0^{(i)},\bm{x}\rangle)| \leq \bar{v}_W\sigma_1|c_0^{(i)}|.
    $$
    For the sub-Gaussian $\psi_2$-norm (see Definition 2.5.6 in \cite{vershynin2018high}), these two inequalities imply
    \begin{align*}
        \|\sigma(\langle W_0^{(i)},\bm{x}\rangle)v_c(u_0^{(i)})\|_{\psi_2} &\leq \frac{\sigma_0\bar{v}_c}{\sqrt{\log 2}},\\
        \|c_0^{(i)}v_W^\top(u_0^{(i)})\bm{x}\sigma'(\langle W_0^{(i)},\bm{x}\rangle)\|_{\psi_2} &\leq C \bar{v}_W\sigma_1
    \end{align*}
    for an absolute constant (Exercise 2.5.8 in \cite{vershynin2018high}). Let
    \begin{equation}
        X_i := \sigma(\langle W_0^{(i)},\bm{x}\rangle)v_c(u_0^{(i)}) + c_0^{(i)}v_W^\top(u_0^{(i)})\bm{x}\sigma'(\langle W_0^{(i)},\bm{x}\rangle),\quad i\in[m].
    \end{equation}
    Then, since $\|\cdot\|_{\psi_2}$ is a norm on sub-Gaussian random variables, we obtain $$\|X_i-\mathbb{E}[X_i]\|_{\psi_2}\leq 2\max\{\sigma_0/\sqrt{\log 2}, C\sigma_1\}\|(\bar{v}_c,\bar{v}_W)\|_2=:K,\quad i\in[m].$$ Thus, by Generalized Hoeffding's inequality for $\psi_2$-bounded random variables (Theorem 2.6.3 in \cite{vershynin2018high}), we obtain
    \begin{equation*}
        \max\Big\{\Big|\frac{1}{m/2}\sum_{i=1}^{m/2}X_i-f^\star(\bm{x})\Big|, \Big|\frac{1}{m/2}\sum_{i=m/2+1}^{m}X_i-f^\star(\bm{x})\Big|\Big\} \leq cK\sqrt{\frac{2\log(2/\delta)}{m}},
    \end{equation*}
    with probability at least $1-\delta$ for some absolute constant $c > 0$. Note that $u_0^{(i)}$ and $u_0^{(i+m/2)}$ are correlated for each $i\in[m/2]$ due to symmetric initialization, and the bound above is divided into two halves to handle this. Using triangle inequality, we obtain
    \begin{equation}
    \Big|\underbrace{\frac{1}{m}\sum_{i=1}^{m}X_i}_{=\varphilin(\bm{x};w_f^\star)}-f^\star(\bm{x})\Big| \leq 2cK\sqrt{\frac{2\log(2/\delta)}{m}}.
    \end{equation}
    with probability at least $1-\delta$. Recall that the bound above holds for $\bm{x}\in\cB_2(0,1)$, thus holds for $\bm{x}_j,~j\in[n]$ individually. By using union bound, we obtain
    \begin{equation}
        \max_{j\in[n]}\Big|\varphilin(\bm{x}_j;w_f^\star)-f^\star(\bm{x}_j)\Big| \leq 2cK \sqrt{\frac{2\log(2n/\delta)}{m}}.
    \end{equation}
    Finally, \eqref{eqn:linearization} implies that
    \begin{equation}
        \max_{j\in[n]}\Big|\varphi(\bm{x}_j;w_f^\star)-f^\star(\bm{x}_j)\Big| \leq 2cK \sqrt{\frac{2\log(2n/\delta)}{m}}+\frac{1}{2}\Lipdf r_{\cC}^2
    \end{equation}
    with probability at least $1-\delta$ over the random initialization. Hence,
    \begin{align}
        \nonumber \er(\bv(w_f^\star))&=\frac{1}{n}\sum_{j=1}^n\ell(\varphi(\bm{x}_j;w_f^\star);\bm{z}_j)\\
        \nonumber &= \frac{1}{n}\sum_{j=1}^n\ell(\varphi(\bm{x}_j;w_f^\star);\bm{z}_j)-\frac{1}{n}\sum_{j=1}^n\ell(f^\star(\bm{x}_j);\bm{z}_j)\\
        \nonumber &\leq \frac{\Lipl}{n}\sum_{j=1}^n |\varphi(\bm{x}_j;w_f^\star)-f^\star(\bm{x}_j)|\\
        \label{eqn:wf-risk-bound}&\leq \Lipl \Big(2cK \sqrt{\frac{2\log(2n/\delta)}{m}}+\frac{1}{2}\Lipdf r_{\cC}^2\Big).
    \end{align}
    Hence, $w_f^\star \in \cC$ and $\er(\bv(w_f^\star)) \leq c'$ with probability at least $1-\delta$ over the random initialization. Since $$\bar{w}\in\arg\min_{w\in\cC}\er(\bv(w)),$$ we have $\er(\bv(\bar{w}))\leq \er(\bv(w_f^\star))$ and therefore 
    \begin{equation}
        \bE_0[\er(\bv(\widehat{w}_k))] \leq \underbrace{\bE_0[\er(\bv(\widehat{w}_k))] - \er(\bv(\bar{w}))}_{\eqref{eqn:avg-iterate-error}} + \underbrace{\er(\bv(w_f^\star))}_{\eqref{eqn:wf-risk-bound}}.
    \end{equation}
    Substituting \eqref{eqn:avg-iterate-error} and \eqref{eqn:wf-risk-bound} into the above inequality concludes the proof.
    
\end{proof}

\section{Proofs for Section \ref{sec:gen}}\label{app:gen}
\begin{proof}[Proof of Lemma 2]
    Consider $\widehat{w}_k,\widehat{w}_k'\in\cC$. Then,
    \begin{align*}
        |\ell(\varphi(\bm{x};\widehat{w}_k);\bm{z})-\ell(\varphi(\bm{x};\widehat{w}_k');\bm{z}))| &\leq \Lipl\Lipf\|\widehat{w}_k-\widehat{w}_k'\|_2\\
        &\leq \Lipl\Lipf\frac{1}{k}\sum_{t=1}^k\|w_t-w_t'\|_2\\
        &\leq \frac{\Lipl\Lipf}{\sqrt{\lambda}}\frac{1}{k}\sum_{t=1}^k\|w_t-w_t'\|_{\bar{\bm{H}}_{t-1}}
    \end{align*}
    where the second line follows from Jensen's inequality and the third line follows from the fact that $\bar{\bm{H}}_k\succeq \lambda\bm{I}$. RHS of the above inequality is independent of $\bm{z}\in\mathscr{Z}$, thus
    $$
        \sup_{\bm{z}\in\mathscr{Z}}|\ell(\varphi(\bm{x};\widehat{w}_k);\bm{z})-\ell(\varphi(\bm{x};\widehat{w}_k');\bm{z}))| \leq \frac{\Lipl\Lipf}{\sqrt{\lambda}}\frac{1}{k}\sum_{t=1}^{k}\|w_t-w_t'\|_{\bar{\bm{H}}_{t-1}}.
    $$
    Taking expectation over the optimization path concludes the proof.
\end{proof}

\begin{proof}[Proof of Theorem 2]
    Let $\pi_k:=\pi_{\cC}^{\BH_k}$, $\pi_k':=\pi_{\cC}^{\BH_k'}$ and $\bar{\pi}_k:=\pi_{\cC}^{\bar{\BH}_k}$ be the projection operators, and define
    \begin{align*}
        u_k&=w_k-\eta\BH_k^{-1}\Psi_k(w_k)\\u_k'&=w_k'-\eta[\BH_k']^{-1}\Psi_k'(w_k').
    \end{align*}
    Then, we have the following error decomposition:
    \begin{align*}
        \|\Delta_{k+1}\|_{\bar{\BH}_k} \leq \|\underbrace{\bar{\pi}_k(u_k-u_k')}_{(A_k)}\|_{\bar{\BH}_k}+\|\underbrace{(\pi_k-\bar{\pi}_k)u_k}_{(B_k)}\|_{\bar{\BH}_k}+\|\underbrace{(\bar{\pi}_k-\pi_k')u_k'}_{(B_k')}\|_{\bar{\BH}_k}.
    \end{align*}
    In this inequality, $(B_k)$ and $(B_k')$ are error terms due to the metric (or projection) mismatch. We show that the critical term $(A_k)$ will yield $\|\Delta_k\|_{\bar{\BH}_{k-1}}$ plus controllable error terms with $n$ and $m$. 
    \paragraph{\textbf{\underline{Bounding $(A_{k,1})$.}}} $\cC$ is compact and convex, and $\bar{\BH}_k$ is positive definite, thus $\bar{\pi}_k$ is non-expansive \citep{nesterov2018lectures, brezis2011functional}:
    \begin{equation*}
        \|\bar{\pi}_k(u_k-u_k')\|_{\bar{\BH}_k} \leq \|u_k-u_k'\|_{\bar{\BH}_k}.
    \end{equation*}
    Let $T_k(w) := w-\eta\bar{\BH}_k^{-1}\Psi_k(w).$ Then, we further decompose $(A_k)$ as follows:
    \begin{align*}
        \|u_k-u_k'\|_{\bbh_k} &\leq \|\underbrace{T_k(w_k)-T_k(w_k')}_{(A_{k,1})}\|_{\bbh_k} + \eta\|\underbrace{\bbh_k^{-1}(\Psi_k(w_k')-\Psi_k'(w_k'))}_{(A_{k,2})}\|_{\bbh_k}
        \\&+ \eta\|\underbrace{(\bbh_k^{-1}-\bm{H}_k^{-1})\Psi_k(w_k)}_{(A_{k,3})}\|_{\bbh_k}+\eta\|\underbrace{\big((\bm{H}_k')^{-1}-\bbh_k^{-1}\big)\Psi'_k(w_k')}_{(A'_{k,3})}\|_{\bbh_k}.
    \end{align*}
    In this decomposition, $(A_{k,3})$ and $(A_{k,3}')$ correspond to the \emph{preconditioner mismatch} terms. 

    The first part is to establish an approximate co-coercivity result in the sense of \cite{baillon1977quelques}.
    \begin{lemma}[Approximate co-coercivity]
        Let $\varepsilon = B\Lipl\Lipdf$ and $\Lambda = 2\varepsilon+B\Lipf^2$. Then, for any $u,v\in\cC$, we have
        \begin{equation}
            \Big(\Psi_k(u)-\Psi_k(v)\Big)^\top (u-v) \geq \frac{1}{2\Lambda}\|\Psi_k(u)-\Psi_k(v)\|_2^2-2\varepsilon\|u-v\|_2^2.
        \end{equation}
        \label{lemma:co-coercivity}
    \end{lemma}
    \begin{proof}[Proof of Lemma \ref{lemma:co-coercivity}] Note the following decomposition of the Hessian:
        \begin{equation}
            \nabla^2\Phi_k(w)=\bm{J}_k^\top(w)\bm{J}_k(w) + \sum_{i\in I_k}\ell'(\varphi(\bm{x}_i;w);\bm{z}_i)\nabla^2\varphi(\bm{x}_i;w).
        \end{equation}
        Since $\bm{J}_k^\top(w)\bm{J}_k(w)\succeq 0$, we have
        \begin{equation*}
            -\varepsilon:=-B\Lipl\Lipdf\leq \lambda_{\min}(\nabla^2\Phi_k(w))\leq \lambda_{\max}(\nabla^2\Phi(w)) \leq B\Lipl\Lipdf=\varepsilon
        \end{equation*}
        by Weyl's inequality \citep{horn2012matrix}. Let $$\Phi_k^\varepsilon(w) = \Phi_k(w) + \frac{1}{2}\varepsilon\|w\|_2^2.$$ Then, the above inequality implies $$\nabla^2\Phi_k^\varepsilon(w) \succeq 0,$$ and we also have $$\nabla\Phi_k^\varepsilon(w) = \nabla\Phi_k(w) + \varepsilon w.$$ For any $u, v\in\cC$, $$\nabla\Phi_k^\varepsilon(u)-\nabla\Phi_k^\varepsilon(v)=\int_0^1\nabla^2\Phi_k^\varepsilon((1-s)u+sv)ds\Big(u-v\Big)=M(u-v).$$ Then,
        \begin{align*}
            \|\nabla\Phi_k^\varepsilon(u)-\nabla\Phi_k^\varepsilon(v)\|_2^2 &= (u-v)^\top M^2(u-v)\\
            &\leq \Lambda(u-v)^\top M(u-v)\\
            &= \Lambda(u-v)^\top \Big(\nabla\Phi_k^\varepsilon(u)-\nabla\Phi_k^\varepsilon(v)\Big)\\
            &= \Lambda (u-v)^\top (\nabla\Phi_k(u)-\nabla\Phi_k(v))+\varepsilon\Lambda \|u-v\|_2^2.
        \end{align*}
        Also, $$\|\nabla\Phi_k(u)-\nabla\Phi_k(v)\|_2^2\leq 2\|\nabla\Phi_k^\varepsilon(u)-\nabla\Phi_k^\varepsilon(v)\|_2^2+2\varepsilon^2\|u-v\|_2^2.$$
        Hence,
        \begin{equation*}
            \frac{1}{2}\|\nabla\Phi_k(u)-\nabla\Phi_k(v)\|_2^2-\varepsilon^2\|u-v\|_2^2 \leq \Lambda (u-v)^\top (\nabla\Phi_k(u)-\nabla\Phi_k(v)) + \varepsilon\Lambda \|u-v\|_2^2,
        \end{equation*}
        and therefore,
        $$
            \|\nabla\Phi_k(u)-\nabla\Phi_k(v)\|_2^2\leq 2\Lambda (u-v)^\top (\nabla\Phi_k(u)-\nabla\Phi_k(v)) + 2\varepsilon(\varepsilon+\Lambda)\|u-v\|_2^2,
        $$
        and equivalently
        \begin{align*}
            (u-v)^\top (\nabla\Phi_k(u)-\nabla\Phi_k(v)) &\geq \frac{1}{2\Lambda } \|\nabla\Phi_k(u)-\nabla\Phi_k(v)\|_2^2-\varepsilon\frac{\varepsilon+\Lambda}{\Lambda}\|u-v\|_2^2\\&\geq 
            \frac{1}{2\Lambda } \|\nabla\Phi_k(u)-\nabla\Phi_k(v)\|_2^2-2\varepsilon \|u-v\|_2^2
        \end{align*}
        since $\Lambda \geq \varepsilon.$ \end{proof}
    Using Lemma \ref{lemma:co-coercivity}, we obtain a non-expansivity result for $T_k$.
    \begin{lemma}[Approximate non-expansivity of $T_k$]
    For any $u,v\in\cC$, if $(\eta,\lambda)$ satisfy the conditions in the paper,
    then we have the (approximate) non-expansivity
    \begin{align*}
        \|T_k(u)-T_k(v)&\|_{\bar{\bm{H}}_k} \leq \|u-v\|_{\bar{\bm{H}}_{k-1}} + 2r_{\cC}\sqrt{\eta\varepsilon}+2\sqrt{\alpha}\Lip_{\nabla\varphi, \cC} B r_{\cC}\Big(\frac{r_{\cC}}{2}+\frac{\sqrt{B}\Lip_{\ell,\cC}}{\mu_0\nu}\Big)
    \end{align*}
    almost surely.
        \label{lemma:contractivity-app}
    \end{lemma}
    \begin{proof}[Proof of Lemma \ref{lemma:contractivity-app}]
        First we make the following decomposition:
        \begin{equation*}
            \|T_k(u)-T_k(v)\|_{\bbh_k}^2=\|u-v\|_{\bbh_k}^2-2\eta(u-v)^\top\Big(\Psi_k(u)-\Psi_k(v)\Big)+\eta^2\|\Psi_k(u)-\Psi_k(v)\|_{\bbh_k^{-1}}^2.
        \end{equation*}
        Using Lemma \ref{lemma:co-coercivity}, we obtain
        \begin{equation}
            \|T_k(u)-T_k(v)\|_{\bbh_k}^2 \leq \|u-v\|_{\bbh_k}^2 - \Big(\frac{\eta}{\Lambda}-\eta^2\Big)\|\Psi_k(u)-\Psi_k(v)\|_2^2+4\eta\varepsilon\|u-v\|_2^2.
        \end{equation}
        By choosing $\eta \leq \frac{1}{2\Lambda}$, the above implies
        \begin{equation}\label{eqn:contr-1}
            \|T_k(u)-T_k(v)\|_{\bbh_k}^2 \leq \|u-v\|_{\bbh_k}^2 - \frac{\eta}{2\Lambda}\|\Psi_k(u)-\Psi_k(v)\|_2^2+4\eta\varepsilon\|u-v\|_2^2.
        \end{equation}
        Note that
        \begin{align}
            \nonumber \|u-v\|_{\bbh_k}^2 &= \|u-v\|_{\bbh_{k-1}}^2 + \alpha\|\bm{J}_k(w_k)(u-v)\|_2^2\\
            &\leq \|u-v\|_{\bbh_{k-1}}^2 + \alpha\|\bm{J}_k(w_k)(u-v)\|_1^2\label{eqn:contr-2}
        \end{align}
        Now,
        \begin{align*}
            \|\Psi_k(u)-\Psi_k(v)\|_2 &\geq \|\bm{J}_k^\top(u)\Big(\bm{G}_k(u)-\BG_k(v)\Big)\|_2-\|\Big(\bm{J}_k(u)-\bm{J}_k(v)\Big)^\top \BG_k(v)\|_2\\
            &\geq \|\bm{J}_k^\top(u)\Big(\bm{G}_k(u)-\BG_k(v)\Big)\|_2-B\Lipdf\Lipl r_{\cC}.
        \end{align*}
        Here, we use Assumption 2,
        $$\|\bm{J}_k^\top(u)\Big(\bm{G}_k(u)-\BG_k(v)\Big)\|_2\|\geq \mu_0\|\bm{G}_k(u)-\BG_k(v)\|_2.$$ We also have
        \begin{equation*}
            \|\BG_k(u)-\BG_k(v)\|_2 \geq \frac{\nu}{\sqrt{B}}\sum_{j\in I_k}|\varphi(\bm{x}_j;u)-\varphi(\bm{x}_j;v)|.
        \end{equation*}
        From Lipschitz-smoothness, 
        \begin{equation*}
            \|\BG_k(u)-\BG_k(v)\|_2 \geq \frac{\nu}{\sqrt{B}}\|\bm{J}_k(u)(u-v)\|_1 -\frac{\nu\sqrt{B}}{2}r_{\cC}^2\Lipdf.
        \end{equation*}
        Hence,
        \begin{align*}
            \|\Psi_k(u)-\Psi_k(v)\|_2 \geq \mu_0\Big(\frac{\nu}{\sqrt{B}}\|\bm{J}_k(u)(u-v)\|_1 -\frac{\nu\sqrt{B}}{2}r_{\cC}^2\Lipdf\Big),
        \end{align*}
        which implies that
        \begin{equation}\label{eqn:contr-3}
            \|\bm{J}_k(u)(v-u)\|_1^2 \leq \frac{2B}{\mu_0^2\nu^2}\|\Psi_k(u)-\Psi_k(v)\|_2^2 + 2\Lipdf^2B^2r_{\cC}^2\Big(\frac{r_{\cC}^2}{4} + \frac{B\Lipl^2}{\nu^2\mu_0^2}\Big).
        \end{equation}
        By substituting \eqref{eqn:contr-2} and \eqref{eqn:contr-3} into \eqref{eqn:contr-1}, we obtain
        \begin{align*}
            \|T_k(u)-T_k(v)\|_{\bbh_k}^2 &\leq \|u-v\|_{\bbh_{k-1}}^2-\Big(\frac{\eta}{2\Lambda}-\frac{2\alpha B}{\mu_0^2\nu^2}\Big)\|\Psi_k(u)-\Psi_k(v)\|_2^2 \\&+ 4\alpha \Lipdf^2B^2r_{\cC}^2\Big(\frac{r_{\cC}^2}{4} + \frac{B\Lipl^2}{\nu^2\mu_0^2}\Big)+4\eta\varepsilon r_{\cC}^2
        \end{align*}
        The choice of $\eta/\alpha$ cancels the second term, thereby concluding the proof.
    \end{proof}

    \paragraph{\textbf{\underline{Bounding $(A_{k,2})$.}}} This term corresponds to \emph{gradient mismatch}. Since $\{j^\star \notin  I_k\}\subset\{\Psi_k(\cdot)=\Psi_k'(\cdot)\}$, we obtain
    \begin{align}
        \nonumber \|\bbh_k^{-1}(\Psi_k&(w_k')-\Psi_k'(w_k'))\|_{\bbh_k}\\ &\nonumber \leq 2\sup_{w\in\cC}\|\Psi_k(w)\|_{\bbh_k^{-1}}\mathbbm{1}_{\{j^\star\in I_k\}}\\
        &\label{eqn:hk-bound-1a}\leq \frac{2B}{\lambda_{\min}^{1/2}(\bbh_k)}\Lip_{\ell,\cC}\Lip_{\varphi,\cC}\mathbbm{1}_{\{j^\star\in I_k\}}.
    \end{align}
    Here, note that $\bE[\mathbbm{1}_{\{j^\star\in I_k\}}]=\cP(j^\star\in I_k) = \frac{B}{n}$ for any $k\in\bN$. In the worst-case scenario, 
    \begin{equation*}
        \lambda_{\min}(\bbh_k) \geq \lambda,
    \end{equation*}
    which implies that 
        $$\|\bbh_k^{-1}(\Psi_k(w_k')-\Psi_k'(w_k'))\|_{\bbh_k} \leq \frac{2B^2}{\lambda^{1/2}n} \Lipl\Lipf.$$

    \textbf{\underline{Bounding $(A_{k,3})$ and $(A_{k,3}')$.}} To bound these error terms that stem from the preconditioner mismatch, we use
    \begin{align}\label{eqn:hk-bound-2a}
        \|(\bbh_k^{-1}-\bm{H}_k^{-1})\Psi_k(w_k)\|_{\bbh_k}\leq \|\bar{\bm{H}}_k^{-1/2}(\bar{\bm{H}}_k-\bm{H}_k)\bm{H}_k^{-1}\|_2\|\Psi_k(w_k)\|_2.
    \end{align}
    We have
    \begin{align*}
        \lambda_{\max}(\bbh_k) &\leq \alpha B (k+1)\Lip_{\varphi,\cC}^2+\lambda,\\
        \sup_{w\in\cC}\|\Psi_k(w)\|_2&\leq B\Lip_{\ell,\cC}\Lip_{\varphi,\cC}
    \end{align*}
    and $$\|\bm{H}_k'-\bbh_k\|_2=\|\bm{H}_k-\bbh_k\|_2=\frac{1}{2}\|\bm{H}_k-\bm{H}_k'\|_2.$$ The following lemma upper bounds the error term $\bE\|\bm{H}_k-\bm{H}_k'\|_2$.
    \begin{lemma}[Stability of $\bm{H}_k$]\label{lemma:H-stability}
        For any $k\in\bN$,
        \begin{align*}
            \bE[\|\bm{H}_k-\bm{H}_k'\|_2] \leq 2B\alpha(k+1) r_{\cC}\Lip_{\varphi,\cC}\Lip_{\nabla\varphi,\cC}+\alpha\Lip_{\varphi,\cC}^2\frac{(k+1)B}{n}.
        \end{align*}
    \end{lemma}
    \begin{proof}[Proof of Lemma \ref{lemma:H-stability}]
        By using the uniform $\Lipdf$-Lipschitz continuity of $\nabla\varphi$, we have 
        \begin{equation*}
            \|\BH_k-\BH_k'\|_2 \leq \alpha \Lipf^2\sum_{t=0}^k\mathbbm{1}_{\{j^\star \in I_t\}} + 2\alpha(k+1)Br_{\cC}\Lipf\Lipdf.
        \end{equation*}
        The result is proved by taking expectation over $I_t$.
    \end{proof}
    \begin{lemma}[Metric mismatch]\label{lemma:metric-mismatch}
        For any $k\in\bN$, we obtain
        \begin{align}
            \nonumber \|(\pi_k&-\bar{\pi}_k)u_k\|_{\bbh_k} \\\nonumber &\leq \|\bbh_k^{-1/2}(\bm{H}_k-\bbh_k)\|_2\sup_{w\in\cC}\|w-u_k\|_2\\
            \label{eqn:hk-bound-3a}&\leq \frac{\|\bm{H}_k-\bm{H}_k'\|_2}{2\lambda^{1/2}_{\min}(\bbh_k)}\sup_{w\in\cC}\|w-u_k\|_2
        \end{align}
        Also, for any $w\in\cC$,
            $\|w-u_k\|_2 \leq  r_{\cC}+\frac{\eta B}{\lambda}\Lip_{\ell,\cC}\Lip_{\varphi,\cC}.$
    \end{lemma}
    \begin{proof}
        First,
        \begin{align*}
            \|w-u_k\|_2 &= \|w-w_k+\eta\bm{H}_k^{-1}\Psi_k(w_k)\|_2\\
            &\leq \|w-w_k\|_2 + \frac{\eta}{\lambda}\|\Psi_k(w_k)\|_2\\
            &\leq r_{\cC} + \frac{\eta}{\lambda}B\Lipl\Lipf.
        \end{align*}
        Secondly, let $\bar{w}_k := \bar{\pi}_ku_k$ and $w_k:=\pi_ku_k$. We will use the following geometric argument due to \cite{bubeck2015convex} (Lemma 3.1): for any $u'\in\bR^p$,
        \begin{align*}
            \big( \pi_ku_k-\bar{\pi}_ku_k \big)^\top \bbh_k\big(u_k-\bar{\pi}_ku_k\big) &\leq 0,\\
            \big( \pi_ku_k-\bar{\pi}_ku_k \big)^\top \BH_k\big(\pi_ku_k-u_k\big) &\leq 0.
        \end{align*}
        Then, by adding and subtraction $\bar{\pi}_ku_k$ in the second inequality,
        \begin{align*}
            \big( \pi_ku_k-\bar{\pi}_ku_k \big)^\top \BH_k\big(\pi_ku_k-\bar{\pi}_ku_k+\bar{\pi}_ku_k - u_k\big) = \|(\pi_k-\bar{\pi}_k)u_k\|_{\BH_k}^2 -((\pi_k-\bar{\pi}_k)u_k)^\top (-\BH_k)(u_k-\bar{\pi}_ku_k)
        \end{align*}
        Adding these two inequalities, we obtain
        \begin{align*}
            \|(\pi_k-\bar{\pi}_k)u_k\|_{\BH_k}^2 &\leq ((\pi_k-\bar{\pi}_k)u_k)^\top (\bar{\bm{H}}_k-\bm{H}_k)(u_k-\bar{\pi}_ku_k)\\
            &\leq ((\pi_k-\bar{\pi}_k)u_k)^\top\bbh_k^{1/2}\bbh_K^{-1/2}(\bbh_k-\BH_k)(u_k-\bar{\pi}_ku_k)\\
            &\leq \|(\pi_k-\bar{\pi}_k)u_k\|_{\bbh_k}\|\bbh_K^{-1/2}(\bbh_k-\BH_k)(u_k-\bar{\pi}_ku_k)\|_2\\
            &\leq \|(\pi_k-\bar{\pi}_k)u_k\|_{\bbh_k}\lambda_{\min}^{-1/2}(\bbh_k)\|\bbh_k-\BH_k\|_2\|u_k-\bar{\pi}_ku_k\|_2.
        \end{align*}
        Hence,
        \begin{align*}
             \|(\pi_k-\bar{\pi}_k)u_k\|_{\BH_k} &\leq \lambda_{\min}^{-1/2}(\bbh_k)\|\bbh_k-\BH_k\|_2\|u_k-\bar{\pi}_ku_k\|_2\\
             &\leq \lambda_{\min}^{-1/2}(\bbh_k)\|\bbh_k-\BH_k\|_2\sup_{w\in\cC}\|u_k-w\|_2.
        \end{align*}
        
    \end{proof}
    Consider an event $\bar{E}$ in the sigma-algebra generated by the subsampling process (i.e., optimization path/history) $\sigma(I_k:k\in\bN)$, which is defined 
    \begin{equation}
        \lambda_{\min}(\bm{H}_t)\geq \lambda_t,~t\in\bN
    \end{equation}
    for a given sequence $\{\lambda_t:t\in\bN\}$. We always have $\inf_{t\in\bN}\lambda_t \geq \lambda > 0$. Now, we summarize the bounds found earlier in this proof as follows:
    \begin{align*}
        \bE[(A_{t,1});\bar{E}] &\leq \bE\big[\|\Delta_t\|_{\bbh_{t-1}} ;\bar{E}]+ \underbrace{2r_{\cC}\sqrt{\eta\varepsilon}+2\sqrt{\alpha}\Lip_{\nabla\varphi, \cC} B r_{\cC}\Big(\frac{r_{\cC}}{2}+\frac{\sqrt{B}\Lip_{\ell,\cC}}{\mu_0\nu}\Big)}_{=:Z_1},\\
        \bE[(A_{t,2});\bar{E}] &\leq \frac{2B^2}{\lambda_t^{1/2}n}\Lipl\Lipf=:\frac{1}{n}\lambda_t^{-1/2}\underbrace{2B^2\Lipl\Lipf}_{=:Z_2},\\
        \bE[(A_{t,3})+(A_{t,3}'); \bar{E}] &\leq \alpha (t+1)\lambda_t^{-3/2}\underbrace{B^2\Lipl\Lipf}_{=:Z_3}\underbrace{\Big(2r_{\cC}\Lipl\Lipdf+\Lipf^2\frac{1}{n}\Big)}_{=:Z_0^{(m,n)}},\\
        \bE[(B_k)+(B_k');\bar{E}] &\leq \alpha (t+1)\lambda_{t}^{-1/2}\underbrace{\Big(r_{\cC}+\frac{\eta B}{\lambda}\Lipl\Lipf\Big)B}_{=:Z_4}\underbrace{\Big(2 r_{\cC}\Lipf\Lipdf+\frac{\Lipf^2}{n}\Big)}_{=Z_0^{(m,n)}}
    \end{align*}
    Consequently, we have
    \begin{equation*}
        \bE[\|\Delta_{t+1}\|_{\BH_t};\bar{E}] \leq \bE[\|\Delta_t\|_{\BH_{t-1}};\bar{E}] + Z_1 + \eta Z_2\lambda_t^{-1/2}\frac{1}{n} + \alpha \frac{t+1}{\lambda_t^{1/2}}\Big(\eta \frac{Z_3}{\lambda_t}+Z_4\Big)Z_0^{(m,n)}.
    \end{equation*}
    Telescoping sum over $t=0,1,2,\ldots,k-1$ yields
    \begin{equation}
        \bE[\|\Delta_k\|_{\BH_{k-1}};\bar{E}] \leq kZ_1 + \frac{\eta Z_2}{n}\sum_{t=0}^{k-1}\lambda_t^{-1/2}+\alpha\eta Z_3Z_0^{(m,n)}\sum_{t=0}^{k-1}\frac{t+1}{\lambda_t^{3/2}} + \alpha Z_4 Z_0^{(m,n)}\sum_{t=0}^{k-1}\frac{t+1}{\lambda_t^{1/2}}.
    \end{equation}
    Now, note that $\bP(\bar{E})=1$ if $\bar{E}=\{\lambda_t = \lambda,~t \geq 0\}$. Then, we obtain the worst-case bound:
    $$
        \bE\|\Delta_k\|_{\BH_{k-1}} \leq kZ_1 + \frac{\eta k Z_2}{\lambda^{1/2}n} + \alpha\eta Z_3Z_0^{(m,n)}\frac{k^2}{\lambda^{3/2}} + \alpha Z_4 Z_0^{(m,n)}\frac{k^2}{\lambda^{1/2}},
    $$
    which concludes the proof.
\end{proof}
\begin{remark}
    Fix $\lambda > 0$, and set $\eta = \frac{C}{k}$ and $\alpha = \frac{\xi C}{k}$ for some $C > 0$ that satisfies the conditions in (8). Then, we obtain
    \begin{equation*}
        \bE\|\Delta_{k}\|_{\BH_{k-1}} \lesssim_C kZ_1 + \frac{Z_2}{\lambda^{1/2}n} + Z_3Z_0^{(m,n)}\frac{1}{\lambda^{3/2}}+\frac{Z_4Z_0^{(m,n)}k}{\lambda^{1/2}}\mbox{ for }k\geq 1.
    \end{equation*}
    Now, note that 
    \begin{align*}
        Z_1 &\lesssim r_{\cC}\frac{1}{\sqrt{k}}\Bigg(\sqrt{B\Lipl\Lipdf}+Br_{\cC}\Lipdf\Big(\frac{r_{\cC}}{2}+\frac{\sqrt{B}\Lipl}{\mu_0\nu}\Big)\Bigg)\lesssim \frac{\sqrt{B}}{\sqrt{k\sqrt{m}}}+\frac{Br_{\cC}}{\sqrt{m}}(r_{\cC}+\sqrt{B}),\\
        Z_0^{(m,n)} &\lesssim r_{\cC}\Lipdf + \frac{1}{n}\lesssim \frac{r_{\cC}}{\sqrt{m}}+\frac{1}{n}.
    \end{align*}
    Substituting these into the stability bound, we obtain
    \begin{align*}
        \bE\|\Delta_{k}\|_{\BH_{k-1}} &\lesssim \frac{\sqrt{B\cdot k}}{m^{1/4}}+\frac{B\sqrt{k}}{\sqrt{m}}(1+\sqrt{B})+\frac{B^2}{\lambda^{1/2}n}+\frac{B^2}{\lambda^{3/2}}\Big(\frac{1}{\sqrt{m}}+\frac{1}{n}\Big)+\frac{k}{\lambda^{1/2}}\Big(\frac{1}{\sqrt{m}}+\frac{1}{n}\Big),\\
        &\lesssim \frac{\sqrt{B k}}{m^{1/4}}+\frac{1}{\sqrt{m}}\Big(B\sqrt{k}+B^{3/2}\sqrt{k}+B^2\lambda^{-3/2}+k\lambda^{-1/2}\Big) + \frac{1}{n}\Big(\frac{B^2}{\lambda^{1/2}}+\frac{B^2}{\lambda^{3/2}}+\frac{k}{\lambda^{1/2}}\Big).
    \end{align*}
    Proposition 2 is proved via exactly the same methodology, whereby we substitute the lower bound for $\lambda_t$ into the stability bound.
\end{remark}

\end{document}